\documentclass[11pt]{article}

\usepackage[final]{acl}

\usepackage{times}
\usepackage{latexsym}

\usepackage[T1]{fontenc}

\usepackage[utf8]{inputenc}

\usepackage{microtype}

\usepackage{inconsolata}

\usepackage{graphicx}
\usepackage[most]{tcolorbox}
\usepackage{multicol,lipsum}
\usepackage{amsmath}    
\usepackage{amssymb}    
\usepackage{bm}         
\usepackage{graphicx}   
\usepackage{booktabs}   
\usepackage{multirow}   
\usepackage{tabularx}
\usepackage[table]{xcolor} 
\usepackage{amsthm}
\usepackage{booktabs}
\newtheorem*{definition}{Definition}
\newtheorem*{proposition}{Proposition}
\newcommand{\N}{\mathbb{N}}

\title{On the Emergence and Test-Time Use of Structural Information in Large Language Models}

\author{
 \textbf{Michelle Chao Chen\textsuperscript{1,2}}\thanks{Correspondence: \texttt{michelle.chao.chen@mailbox.org}},
 \textbf{Moritz Miller\textsuperscript{1,2}},
 \textbf{Bernhard Sch\"olkopf\textsuperscript{1,2}},
 \textbf{Siyuan Guo\textsuperscript{2,3}}
\\
 \textsuperscript{1}ETH Zurich,
 \textsuperscript{2}Max Planck Institute for Intelligent Systems,
  \textsuperscript{3}University of Cambridge
}
\usepackage{caption}
\usepackage{subcaption}

\begin{document}
\maketitle
\begin{abstract}
Learning structural information from observational data is central to producing new knowledge outside the training corpus. This holds for mechanistic understanding in scientific discovery as well as flexible test-time compositional generation. We thus study how language models learn abstract structures and utilize the learnt structural information at test-time. To ensure a controlled setup, we design a natural language dataset based on linguistic structural transformations. We empirically show that the emergence of learning structural information correlates with complex reasoning tasks, and that the ability to perform test-time compositional generation remains limited. 
\end{abstract}

\begin{figure*}[t]
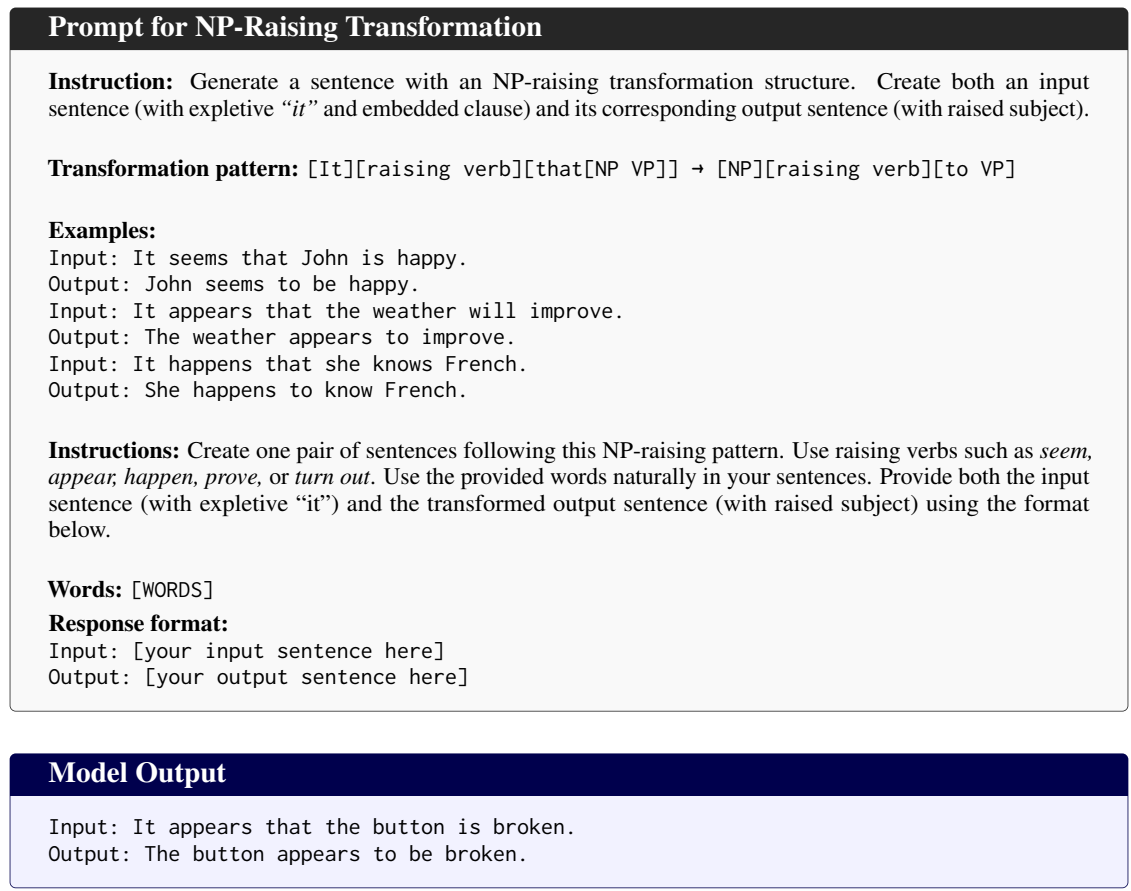

\centering
\begin{tcolorbox}[colback=gray!5!white, colframe=gray!30!black,
                  width=0.92\textwidth, boxrule=0.3pt, arc=2pt,
                  title=\textbf{Prompt for NP-Raising Transformation}]
\small
\textbf{Instruction:} Generate a sentence with an NP-raising transformation structure. 
Create both an input sentence (with expletive \textit{``it''} and embedded clause) and its corresponding output sentence (with raised subject).\\[0.3em]

\textbf{Transformation pattern:} 
\texttt{[It][raising verb][that[NP VP]] → [NP][raising verb][to VP]}\\[0.4em]

\textbf{Examples:}\\
\texttt{Input: It seems that John is happy.}\\
\texttt{Output: John seems to be happy.}\\
\texttt{Input: It appears that the weather will improve.}\\
\texttt{Output: The weather appears to improve.}\\
\texttt{Input: It happens that she knows French.}\\
\texttt{Output: She happens to know French.}\\[0.3em]

\textbf{Instructions:} Create one pair of sentences following this NP-raising pattern. 
Use raising verbs such as \textit{seem, appear, happen, prove,} or \textit{turn out}. 
Use the provided words naturally in your sentences. Provide both the input sentence (with expletive ``it'') and the transformed output sentence (with raised subject) using the format below.\\[0.3em]

\textbf{Words:} \texttt{[WORDS]}\\[0.3em]
\textbf{Response format:}\\
\texttt{Input: [your input sentence here]}\\
\texttt{Output: [your output sentence here]}
\end{tcolorbox}

\vspace{0.2em}

\begin{tcolorbox}[colback=blue!5!white, colframe=blue!30!black,
                  width=0.92\textwidth, boxrule=0.3pt, arc=2pt,
                  title=\textbf{Model Output}]
\small
\texttt{Input: It appears that the button is broken.}\\
\texttt{Output: The button appears to be broken.}
\end{tcolorbox}

\caption{Example of a model prompt for the NP-raising transformation task (top) and a sample output generated by the model (bottom).}
\label{fig:np_raising_prompt}
\end{figure*}

\section{Introduction}
Nature consists of structure and scale, building a hierarchical level of abstract structures to enable responsive and adaptive systems. \citet{guo2025physics} postulates that learning, too, follows the laws of physics under the principle of least action. This work examines how structural information emerges under current learning paradigms and evaluates models’ compositional generalization at test time. The goal is to provide evidence about their ability to generate genuinely new knowledge that is not present in the training corpus.

With the successful application of large models on language, to study structure, we use linguistic structure as a synthetic playground to understand how LLMs learn and compose structure to enable adaptable prediction. One interesting aspect of language comprehension is the ability to compose existing building blocks to comprehend unseen new combinations. This paper studies the problem \textit{how} machines comprehend structural information and \textit{whether} they can utilize learnt knowledge for test-time composition.

To ensure a controlled synthetic playground, we generate a natural language dataset based on \textit{Transformational Grammar (TG)} \citep{chomsky57syntactic,radford1988transformational}. This allows us to analyze \textit{whether} and \textit{how} the model learns the emergence of structure during training, analyze \textit{whether} they can compose learnt structures at test-time, and provide evidence on where in the model this behavior occurs. By doing so, we shine light on how LLMs can generate sentences beyond those directly observed in the corpus. Our contributions are:
\begin{itemize}
    \item We introduce a natural language dataset based on linguistic structural transformations to formally study structural information in language (Section \ref{subsec:dataset}).
    \item We show empirical results that (1) the emergence of structural information during learning correlates with the success of complex reasoning tasks (Section \ref{subsec:emergence}), and (2) the test-time compositional generation is still limited under fine-tuning (Section \ref{subsec:compositional}). 
    \item We perform ablation studies to identify which network components are responsible for learning structural information (Section \ref{subsec:ablation}). 
\end{itemize}

\section{Related Work}

\paragraph{Compositional Generalization.}
Research from cognitive sciences and linguistics argues that humans generalize understanding in language to unseen concepts by interpreting known components, i.e., primitives, and reorganizing basic building blocks to comprehend unseen combinations~\citep{chomsky57syntactic,montague70universal,montague74selected,fodor88connectionism}. Such generalization ability is also advantageous for machines to have in terms of prediction robustness under out-of-distribution generalization and, more importantly, generation of new knowledge not existed in the training corpora. \citet{lake18generalization}, for example, introduces the SCAN dataset on which models are asked to translate natural language commands into primitive-based action protocols. Despite near-perfect in-distribution performance, it is difficult for GRUs~\citep{chung2014gru}, LSTMs~\citep{hochreiter97lstm} and Transformers~\citep{vaswani17transformers} to perform structural generalization to unseen concepts~\citep{kim20cogs}. Putting this into perspective,~\citet{ontanon2022making} observes improving structural generalization capabilities by varying the model design and~\citet{wold2025systematic} quantifies the difficulty of systematic generalization by introducing an entropy-based measure.



\paragraph{Mechanistic Interpretability.} Controlled linguistic generation in language has also been used to probe and understand the inner workings of Transformers. \citet{allen23physics}, for example, studies how Transformers learn hierarchical language structure through synthetic generation of CFGs. The authors discover an implicit dynamic programming algorithm in the model that enables correct next-token prediction. Others probe the residual stream for grammatical number~\citep{lasri22probing,ferrando24similarity} or to address the binding problem~\citep{feng24how}. Relevant to our methodological setup, \citet{mccoy20syntax} studies the sequence-to-sequence model's~\citep{sutskever14sequence} ability to transform a declarative statement into a question through the Question Formation Task. This setup is akin to Transformational Grammar~\citep{chomsky57syntactic,radford1988transformational} which operates over derived structures instead of hierarchical configurations as CFGs do. 

\paragraph{Causality.} A probabilistic formal treatment of structure and its advantages is manifested in causal studies \citep{PetJanSch17, Pearl_2009}. Learning structural knowledge from either observational or interventional data enables downstream tasks such as causal effect estimation \citep{Guoetal24, robertson2025pfn}, counterfactual reasoning \citep{miller2025counterfactual}, robust prediction under out-of-distribution shifts \citep{GuoWilSch24, PerKugSch22}, and representation and world model understanding \citep{Reizingeretal25, lei2022variational, reizinger2025skill}. It has long been established that structural knowledge is impossible to discover from observational data alone \citep{Pearl_2009}; recent work \citep{GuoTotSchHus23} show that heterogeneous observational data is capable of recovering unique causal structures. We thus hypothesize that the model learning from a diverse training corpus also learns implicit re-usable linguistic structures that would be useful for (1) mechanistic understanding; (2) robust composition for new knowledge generation.

\section{Problem Motivation}

Transformational grammar posits that sentences have two levels of representation: a deep structure, which captures the underlying semantic and syntactic relations between elements of a sentence, and a surface structure, which represents the realized or spoken form \cite{radford1988transformational}. Transformational rules operate on the deep structure to derive the surface structure, thereby explaining how speakers can generate and comprehend an unlimited number of grammatical sentences.

We adopt this perspective to examine the mechanisms that allow language models to exhibit similar generative capabilities. Our dataset contains input sentences paralleling deep structures and output sentences produced through custom transformation rules designed to achieve analogous effects.

\subsection{Dataset}
\label{subsec:dataset}
To allow a rigorous study of how models learn, represent, and apply linguistic transformations, we design a dataset based on Transformational Grammar. 

\begin{definition}
    Let $s \in \mathcal{S}$ denote a natural language sentence for $\mathcal{S}$ the space of all sentences. Let $T \in \N$ define the number of possible transformations on the input sentence $s$, and by $R \geq T$ the number of mutually disjoint sets on which at least one of the $T$ transformations can be performed. We partition $\mathcal{S}$ into $R + 1$ disjoint sets and define the sequence of disjoint sets $(\Omega_r)_{r \in \{1,...,R + 1\}}$.
    For $r \in \{1,...,R\}$, we finally define transformations $F_r: 
    \Omega_r \rightarrow \mathcal{S}$.
\end{definition}

\begin{proposition}    
    In our setup, any transformation on sentence $s \in \mathcal{S}$ satisfies either of the two cases:
    \begin{enumerate}
        \item for $r \in \{1,...,R\}$, $F_r: \Omega_r \rightarrow \Omega_{r'}$ such that $F_r(s) \in \Omega_{r'}$ for $r' \in \{1,...,R\}$,
        \item for $r \in \{1,...,R\}$, $F_r: \Omega_r \rightarrow \Omega_{R + 1}$ such that $F_r(s) \in \Omega_{R + 1}$.
    \end{enumerate}
    \label{prop:twocases}
\end{proposition}
In our dataset, each $F_r$ corresponds to a specific transformation, such as passivization, NP-raising, or question formation. For example, a sentence $s \in \Omega_r$ like ``The scientist discovered the formula'' may be transformed by $F_{\text{passive}}$ into ``The formula was discovered by the scientist,'' which lies in another subset $\Omega_{r'}$. Since multiple transformations can be applied sequentially (e.g., applying NP-raising after passivization), the output of one transformation can serve as the input to another. However, certain transformations eventually reach an \emph{absorbing state} $\Omega_{R+1}$—a set of sentences on which no further valid transformations can be applied under our defined grammar (e.g., once a sentence is already in passive voice, applying passivization again has no effect).

\medskip
\noindent
\textbf{Dataset Features.} 
A summary of the dataset is provided in Table~\ref{fig:dataset_summary}.
\begin{itemize}
    \item \textbf{Single-level transformations.} A transformation operator $\mathcal{A}$ is applied to a base sentence $s$, producing $\mathcal{A}(s)$.
    \item \textbf{Nested transformations.} Multiple operators are applied sequentially, e.g., $\mathcal{B}(\mathcal{A}(s))$, where the output of one transformation becomes the input to the next. For instance, starting from the base sentence ``The artist applied the makeup for the photoshoot,'' applying passivization yields ``The makeup was applied for the photoshoot by the artist,'' and subsequently applying question formation produces ``Was the makeup applied for the photoshoot by the artist?''.
    \item \textbf{Compositional ambiguity.} Not all transformations can be meaningfully composed: applying I-Movement (question formation) after Extraposition can yield ungrammatical outputs (e.g., transforming ``The book disappeared on the table” into ``Did the book disappeared on the table?” violates tense agreement). Thus, models must implicitly identify compatible patterns before applying multiple transformations.
\end{itemize}
In total, we define ten distinct transformation types grouped into five broader syntactic categories, such as movement, raising, and passivization. Our dataset includes several types of transformation structures (see Table \ref{tab:transformations}). 

\medskip \noindent
\textbf{Dataset Construction.} For sentence generation, we use DeepSeek-V3 \citep{deepseekai2025deepseekv3technicalreport} with instructions. To improve consistency in generation, we prompt the model to produce input–output pairs together rather than generating base sentences first and then applying transformations separately. Prompts include high-quality examples, explicit descriptions of the transformation to be applied, and a set of required words sampled from the TinyStories \citep{eldan2023tinystoriessmalllanguagemodels} vocabulary to increase lexical variety. An example of such prompts is shown in Figure~\ref{fig:np_raising_prompt}. We generate roughly $2000$ samples for each single-level transformation and $500$ samples for each nested transformation and perform filtering for duplicates and low-quality examples containing repetitions or irrelevant tokens.

\begin{table*}[t]
\centering
\caption{Overview of the major syntactic transformations represented in our dataset. Each transformation type is paired with a short linguistic description and an illustrative example showing a base sentence $s$ and its transformed version $\mathcal{T}(s)$.}
\resizebox{\textwidth}{!}{
\begin{tabular}{l p{6.5cm} p{7cm}}
\toprule
\textbf{Transformation} & \textbf{Description} & \textbf{Example} \\
\midrule
Extraposition & Moves prepositional phrases from within noun phrases to sentence-final position &
\textbf{Base} $s$: The book on the table disappeared.\\[-1.0em]
& & \textbf{Output} $\mathcal{T}(s)$: The book disappeared on the table. \\[1.0em]
I-Movement & Moves auxiliaries or modals to sentence-initial position to form questions &
\textbf{Base} $s$: She can swim.\\[-1.0em]
& & \textbf{Output} $\mathcal{T}(s)$: Can she swim? \\[1.0em]
NP Passive & Converts active sentences to passive with clear subject–object alternation &
\textbf{Base} $s$: The scientist discovered the formula.\\[-1.0em]
& & \textbf{Output} $\mathcal{T}(s)$: The formula was discovered by the scientist. \\[1.0em]
NP Raising & Converts expletive or embedded clauses to raised-subject structures &
\textbf{Base} $s$: It seems that John is honest.\\[-1.0em]
& & \textbf{Output} $\mathcal{T}(s)$: John seems to be honest. \\[1.0em]
V-Movement & Integrates separated infinitive or auxiliary components into a single clause &
\textbf{Base} $s$: The children; to play outside.\\[-1.0em]
& & \textbf{Output} $\mathcal{T}(s)$: The children play outside. \\
\bottomrule
\end{tabular}
}
\label{tab:transformations}
\end{table*}

\section{Experiments}
We use our dataset to investigate three questions:
\begin{itemize}
    \item The emergence of structural understanding during learning (see Section \ref{subsec:emergence})
    \item The ability to compose structures in generation (see Section \ref{subsec:compositional})
    \item Ablation studies on model component contribution (see Section \ref{subsec:ablation})
\end{itemize}

\subsection{Emergence of Structures}
\label{subsec:emergence}
The first group of experiments investigates \textit{whether} and \textit{when} models learn grammatical structures in their sentence representations.  

\textbf{Methodology.} A sentence with $n$ tokens is represented as $s = (t_1, t_2, \dots, t_n)$. Residual stream activations at layer $\ell$ denotes as  $h^{(\ell)}_{i} \in \mathbb{R}^d$.  From the last layer $L$, we obtain a sentence-level representation by mean pooling across the token-level
$
\bar{h}(s) = \frac{1}{n} \sum_{i=1}^{n} h^{(L)}_{i}.
$ For a transformation $\mathcal{A}(\cdot)$ applied to base sentence $s$, we compute a difference vector
$
\Delta_A(s) = \bar{h}(A(s)) - \bar{h}(s),
$
and an $\ell_2$ distance between the base and transformed sentence to capture the structural information
$
d(s, A(s)) = \lVert \bar{h}(A(s)) - \bar{h}(s) \rVert_2.
$

\textbf{Experiments.}
\textit{(1) Visualization.} We cluster the $\Delta_A(s)$ vectors across transformations using K-means clustering with $k=10$ transformation types and apply dimensionality reduction via PCA and t-SNE \cite{tsne} for visualization. We compute cosine similarities between different vectors and analyze separability scores to quantify how distinctly different transformation types cluster in representation space. \textit{(2) Structure emergence.} Using $d$ as a metric for structural understanding in latent representations, we analyze its evolution across Pythia-410M \cite{pythia} checkpoints sampled at approximately exponentially spaced training steps throughout training. To assess the relationship between syntactic representation development and language modeling capability, we evaluate each checkpoint on two complementary metrics: (1) perplexity on pure language modelling tasks, such as WikiText-103 \cite{wikitext} and Paloma benchmarks \cite{paloma}, and (2) accuracy on standardized benchmarks from the LM Evaluation Harness \cite{eval-harness} which focus more on natural language reasoning and knowledge.

\begin{figure}[t]
    \centering
    \includegraphics[width=\columnwidth]{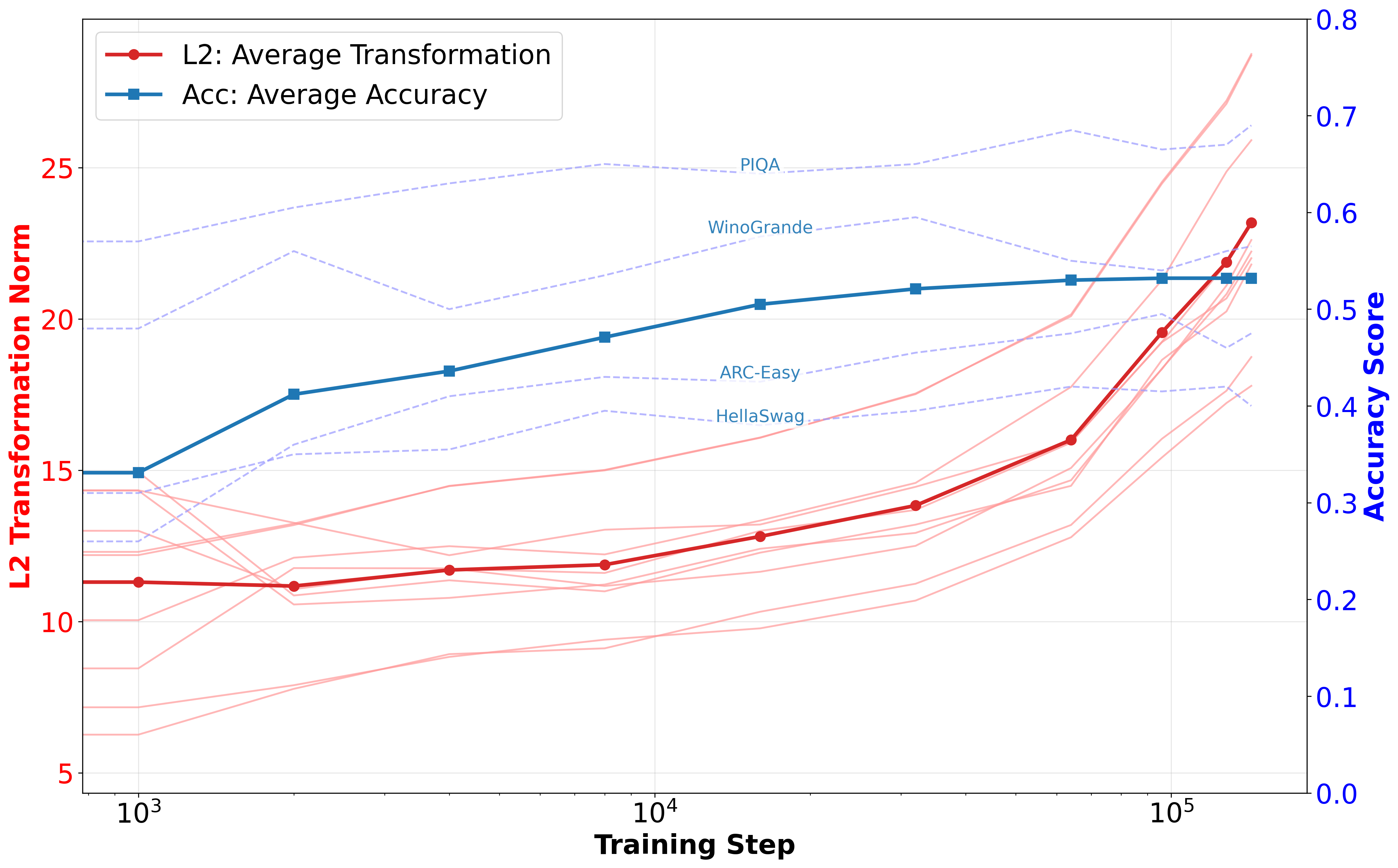}
    \vspace{0.3em} 
    \includegraphics[width=\columnwidth]{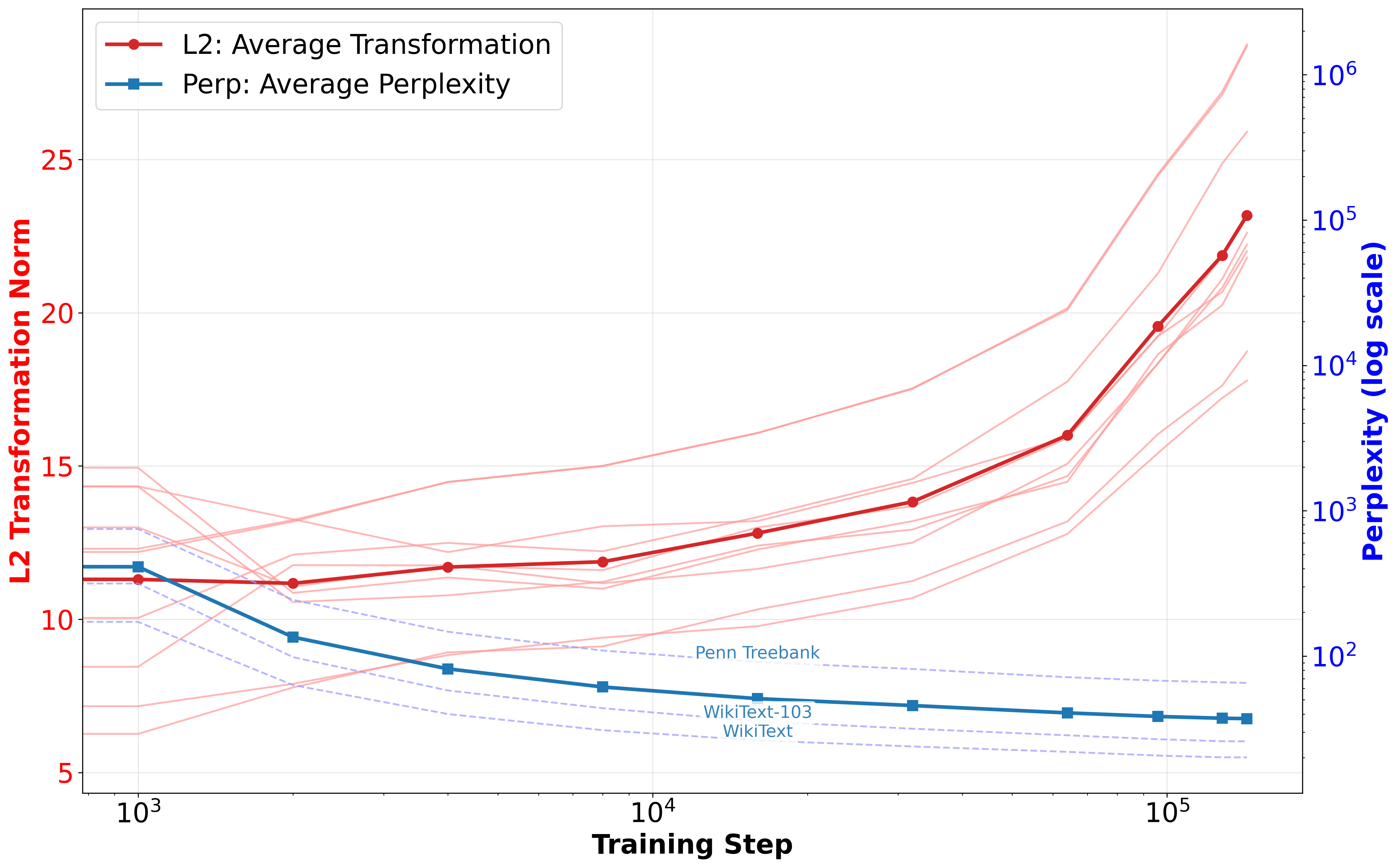}
    \caption{L2 transformation norms and Pythia-410M performance training. Red lines show syntactic transformation embedding differences. Blue lines show model performance: (Top) downstream task accuracy on reasoning tasks. (Bottom) language modeling perplexity on WikiText and Paloma datasets. Individual metrics shown as dashed lines, averages as solid lines.}
    \label{fig:pythia_l2}
\end{figure}

\textbf{Results.} \textit{(1) Development of Internal Representations.}  
We observe that different sentence transformations form distinct and separable clusters in the embedding space (see Figure~\ref{fig:pca_clustering_progression}). This pattern appears consistently across both PCA projections and t-SNE visualizations as seen in Figure~\ref{fig:tsne_clustering_progression}, indicating that the model captures shared structural regularities. Furthermore, embeddings of transformations in the same category (e.g., NP Passive) form clusters close to each other, suggesting that the model consistently encodes these structural variations.

\textit{(2) Emergence of structure.} Most notably in Figure~\ref{fig:pythia_l2}, the average L2 norm between base and transformed sentence representations increases steadily throughout training, reflecting that the model becomes more sensitive to syntactic alterations. A sharp rise occurs around step 64k, which likely marks a phase transition where structural distinctions become more pronounced. 

\textit{(3) Correlation with General Language Modeling Performance.}  
To examine how representational changes relate to general modeling ability, we compare these findings to standard performance metrics. Figure~\ref{fig:pythia_l2} (bottom) shows that perplexity on language modeling tasks decreases sharply during early training steps. On the LM Harness benchmark, the Pythia-410M model shows limited improvement on causal or physical reasoning tasks—expected given its smaller size—but the L2 norm trends align more closely with gains in core language modeling performance, such as next-word prediction. This suggests that the refinement of structural representations coincides with the model’s development of linguistic competence.

\textit{Overall Insight.}  
These results indicate that sentence representations exhibit a sudden phase transition as training progresses. The perplexity loss does not correlate closely with the structural information learned, but the structural information correlates with the more complex reasoning tasks, indicating that reasoning may require structural information.

\subsection{Structural Compositional Generalization}
\label{subsec:compositional}
To evaluate a trained model's ability to use the internal representation of the transformations, we perform both full parameter finetuning and parameter-efficient finetuning using LoRA adapters \citep{hu2022lora} with LLaMA3-8B~\citep{grattafiori2024llama3herdmodels}. 

In particular, given a base sentence $s$ and transformation types $\mathcal{A}$ and $\mathcal{B}$, we train the model to produce the correct transformed output $\mathcal{T}(s)$ from prompts of the form shown in Figure~\ref{fig:prompt_formats}. For nested transformations, the prompt specifies the combined transformation label (e.g., \texttt{Transform ($\mathcal{A}$+$\mathcal{B}$)}). During fine-tuning, both intermediate (i.e. application of one transformation) and final outputs are provided. At inference, only the initial input and the transformation name are given, and the model must generate both the intermediate and final transformed sentences.

\begin{figure}[t]

\small
\textbf{Fine-tuning Prompt Example}\\[0.3em]
During fine-tuning, both the transformation instruction and the target output are provided to the model. The model learns to predict the transformed output.

\begin{tcolorbox}[colback=gray!5!white, colframe=gray!35!black, boxrule=0.3pt, arc=2pt, width=0.95\columnwidth]
\texttt{Transform (A): The teacher graded the exams.}\\
\texttt{Output: The exams were graded by the teacher.}
\end{tcolorbox}

\textbf{Evaluation Prompt}\\[0.3em]
At inference, only the transformation instruction and input sentence are provided. The model must generate the corresponding output.

\begin{tcolorbox}[colback=gray!5!white, colframe=gray!35!black, boxrule=0.3pt, arc=2pt, width=0.95\columnwidth]
\texttt{Transform (A): The teacher graded the exams.}\\
\texttt{Output:}
\end{tcolorbox}
\caption{Prompt formats used during fine-tuning and inference.}
\label{fig:prompt_formats}
\end{figure}

\textbf{Evaluation Protocol.} We assess performance using exact match accuracy and partial match accuracy (word overlap using Jaccard similarity $\ge 0.8$). For nested transformations producing multiple outputs, we check whether the prediction contains all expected results. The evaluation covers both \textit{in-distribution} transformation types and crucially tests \emph{compositional generalization} to unseen nested transformations (e.g., training on $\mathcal{A}(\mathcal{B}(s))$ and $\mathcal{B}(\mathcal{C}(s))$, then testing on $\mathcal{A}(\mathcal{C}(s))$).

\textbf{Disentanglement Hypothesis.} We specifically investigate whether models can perform nested transformations \emph{without observing intermediate results} during training. This is based on the theoretical evidence that structure is discoverable from observational data \citep{GuoTotSchHus23} and the hypothesis that next-token prediction performs implicit disentanglement automatically \citep{zhang2024memory}. Thus we finetune an additional model where the intermediate results of nested transformations are not shown. In particular, a model trained only on a diverse set of input-output pairs $(s, \mathcal{A}(\mathcal{B}(s)))$, $(s, \mathcal{B}(\mathcal{C}(s)))$ can successfully generate $\mathcal{A}(\mathcal{C}(s))$ for novel combinations. This suggests the model learns to disentangle and compose the individual transformation operations $\mathcal{A}$ and $\mathcal{C}$ at test-time, rather than merely memorizing specific transformation sequences. We believe this is one of the desirable features of next-generation AI capability and we ask the question whether the learned representations $\Delta A(s)$ and $\Delta C(s)$ support true compositional generation at test-time.

\begin{table*}[t]
\centering
\small
\caption{
\textbf{Comparison of fine-tuning methods across transformation settings.}
We report exact and partial match accuracy for three model variants: 
(1) regular full-parameter fine-tuning, 
(2) LoRA fine-tuning, and 
(3) LoRA fine-tuning on sequences without intermediate results.
Results highlight the generalization gap across transformation complexity and out-of-distribution (OOD) settings.
}
\setlength{\tabcolsep}{8pt}
\renewcommand{\arraystretch}{1.1}
\begin{tabular}{lcccccc}
\toprule
\multirow{2}{*}{\textbf{Dataset / Model Variant}} &
\multicolumn{2}{c}{\textbf{Full parameter}} &
\multicolumn{2}{c}{\textbf{LoRA }} &
\multicolumn{2}{c}{\textbf{LoRA (No-Intermediate)}} \\
\cmidrule(lr){2-3} \cmidrule(lr){4-5} \cmidrule(lr){6-7}
 & \textbf{Exact} & \textbf{Partial} & \textbf{Exact} & \textbf{Partial} & \textbf{Exact} & \textbf{Partial} \\
\midrule
\textbf{Single Transformations}         & 54.80\% & 60.10\% & 96.40\% & 98.80\% & 96.40\% & 98.60\% \\
\textbf{Double Transformations}         & 37.20\% & 40.80\% & 0.00\%  & 86.95\% & 0.00\%  & 81.89\% \\
\textbf{Double w/ Intermediate}         & -- & -- & 82.74\% & 84.56\% & 74.37\% & 75.79\% \\
\textbf{OOD (A+C, H+E combinations)}    & -- & -- & 0.00\%  & 45.83\% & 0.00\%  & 49.25\% \\
\textbf{OOD w/ Intermediate (A+C, H+E)} & -- & -- & 4.92\% & 9.85\%  & 8.54\% & 32.16\% \\
\bottomrule
\end{tabular}
\label{tab:finetuning_comparison}
\end{table*}

\textbf{Results.} \textit{(1) Full-Parameter Fine-tuning.}
Full-parameter fine-tuning performed poorly, likely due to the limited dataset size leading to unstable optimization and overfitting. The model failed to generalize to transformations seen during training, thus we focused subsequent experiments, including OOD evaluations, on LoRA-based fine-tuning.

\textit{(2) LoRA Fine-tuning.}
Using Low-Rank Adaptation (LoRA), we fine-tuned 0.52\% of the parameters with significant improvement from full parameter fine-tuning as shown in Table~\ref{tab:finetuning_comparison}. On the validation set, the model achieves high partial-match accuracy but rarely produces exact matches for nested transformations when the intermediate step (i.e., the result of the first transformation) is not shown. This suggests that while the model captured aspects of the transformation rules, it does not fully disentangle the transformation operations, which is further supported by the evaluation results on the unseen transformations (OOD). 

\textit{(3) LoRA Fine-tuning with Truncated Dataset.}
To further probe the ability to perform structural composition, we check the effect of providing the first intermediate result (e.g., $\mathcal{B}(s)$) for novel transformation pairs. Performance improved in this setting, indicating that the model benefits from intermediate structural guidance during multi-step transformations.

In the variant of the dataset, the intermediate results for nested transformations are removed, making them appear as single-step transformations distinguishable only by the transformation name at the beginning of the prompt. The LoRA model trained on this dataset achieves comparable performance on both validation and OOD sets when intermediate results are available during training.

\textit{Overall Insight.}
These findings suggest that the model fine-tuned in a small data regime cannot flexibly re-compose out-of-distribution to produce new knowledge. Nested transformations can, to some extent, be treated as extended single-step transformations. The model’s performance shown in row 3 and 5 on Table~\ref{tab:finetuning_comparison} supports the idea that providing explicit intermediate steps during training and inference aids generalization across compositional structures to a small extent but not decisively.

\subsection{Ablation studies on attention heads and MLPs}
\label{subsec:ablation}
The final set of experiments employs causal intervention analysis to identify which network components are most responsible for syntactic transformations. We perform systematic ablation studies by zeroing the outputs of individual attention heads and MLP blocks, then measuring the resulting changes in next-token probability distributions \cite{nanda2022transformerlens}.

Formally, let $\mathbf{s} = (t_1, \ldots, t_n)$ denote an input sequence, and let $t^*$ be the target token representing the correct syntactic transformation. For any component $c$ (either attention head $(l,h)$ or MLP block $l$), we define:
\begin{align}
p_{\text{clean}}(t^* \mid \mathbf{s}) &= \text{softmax}(\mathbf{W}_U \mathbf{h}_n^{(L)})_{t^*} \\
p_{\text{ablated}}^{(c)}(t^* \mid \mathbf{s}) &= \text{softmax}(\mathbf{W}_U \tilde{\mathbf{h}}_n^{(L)})_{t^*}
\end{align}
where $\mathbf{h}_n^{(L)}$ is the final hidden state, $\tilde{\mathbf{h}}_n^{(L)}$ is the hidden state under intervention on component $c$, and $\mathbf{W}_U$ is the unembedding matrix. The causal contribution of component $c$ is then:
\[
\Delta p^{(c)} = p_{\text{clean}}(t^* \mid \mathbf{s}) - p_{\text{ablated}}^{(c)}(t^* \mid \mathbf{s})
\]

To trace the evolution of transformation-relevant information through the network, we apply layer-wise decoding by computing intermediate probability distributions:
\[
p^{(l)}(t^* \mid \mathbf{x}) = \text{softmax}(\mathbf{W}_U \text{LayerNorm}(\mathbf{h}_n^{(l)}))_{t^*}
\]
for each layer $l \in \{0, 1, \ldots, L\}$, where $\mathbf{h}_n^{(0)}$ represents the embedding layer output and $\mathbf{h}_n^{(l)}$ is the residual stream after layer $l$.

By analyzing the distribution of $\Delta p^{(c)}$ values across all components and the trajectory of $p^{(l)}(t^* \mid \mathbf{x})$ across layers, we intend to identify the specific attention heads and MLP blocks most critical for each syntactic transformation, as well as the layers where transformation-relevant computations primarily occur.

In addition, we construct linear probes to identify directions in the residual stream that differentiate between transformation types. For each transformation type $T_n$ and layer $i$, we use Linear Discriminant Analysis (LDA) to find an optimal separating direction. Specifically, we treat this as a binary classification problem: transformation $T_n$ versus all other transformations $T_{m \neq n}$.

Let $\mathbf{r}_{T_n,i}^{(j)}$ denote the last-token residual activation at layer $i$ for the $j$-th example of transformation type $T_n$, and let $\mathbf{r}_{\neg T_n,i}^{(k)}$ denote the last-token residual activation at layer $i$ for the $k$-th example of any other transformation type. The LDA probe direction $\mathbf{v}_{T_n,i}$ is computed as:

\begin{equation}
\mathbf{v}_{T_n,i} = \mathbf{S}_W^{-1}(\boldsymbol{\mu}_{T_n,i} - \boldsymbol{\mu}_{\neg T_n,i})
\end{equation}

where $\boldsymbol{\mu}_{T_n,i} = \frac{1}{N_{T_n}}\sum_{j=1}^{N_{T_n}} \mathbf{r}_{T_n,i}^{(j)}$ and $\boldsymbol{\mu}_{\neg T_n,i} = \frac{1}{N_{\neg T_n}}\sum_{k=1}^{N_{\neg T_n}} \mathbf{r}_{\neg T_n,i}^{(k)}$ are the mean residual activations for transformation $T_n$ and all other transformations, respectively, and $\mathbf{S}_W$ is the pooled within-class covariance matrix. We train these probes on the training set and evaluate their discriminative power on the test set.

\begin{figure}[t]
    \centering
    \includegraphics[width=\columnwidth]{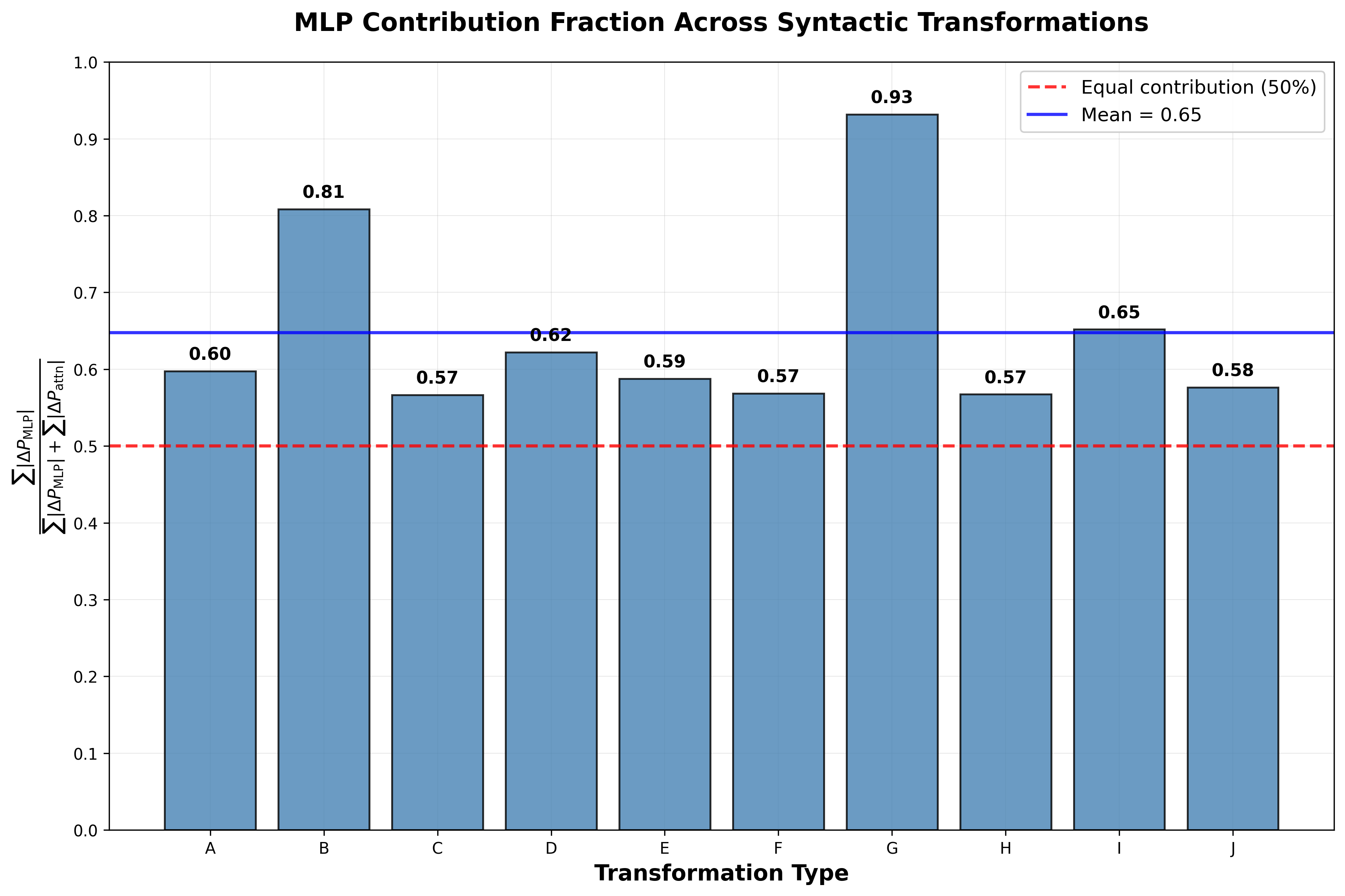}
    \vspace{0.3em} 
    \includegraphics[width=\columnwidth]{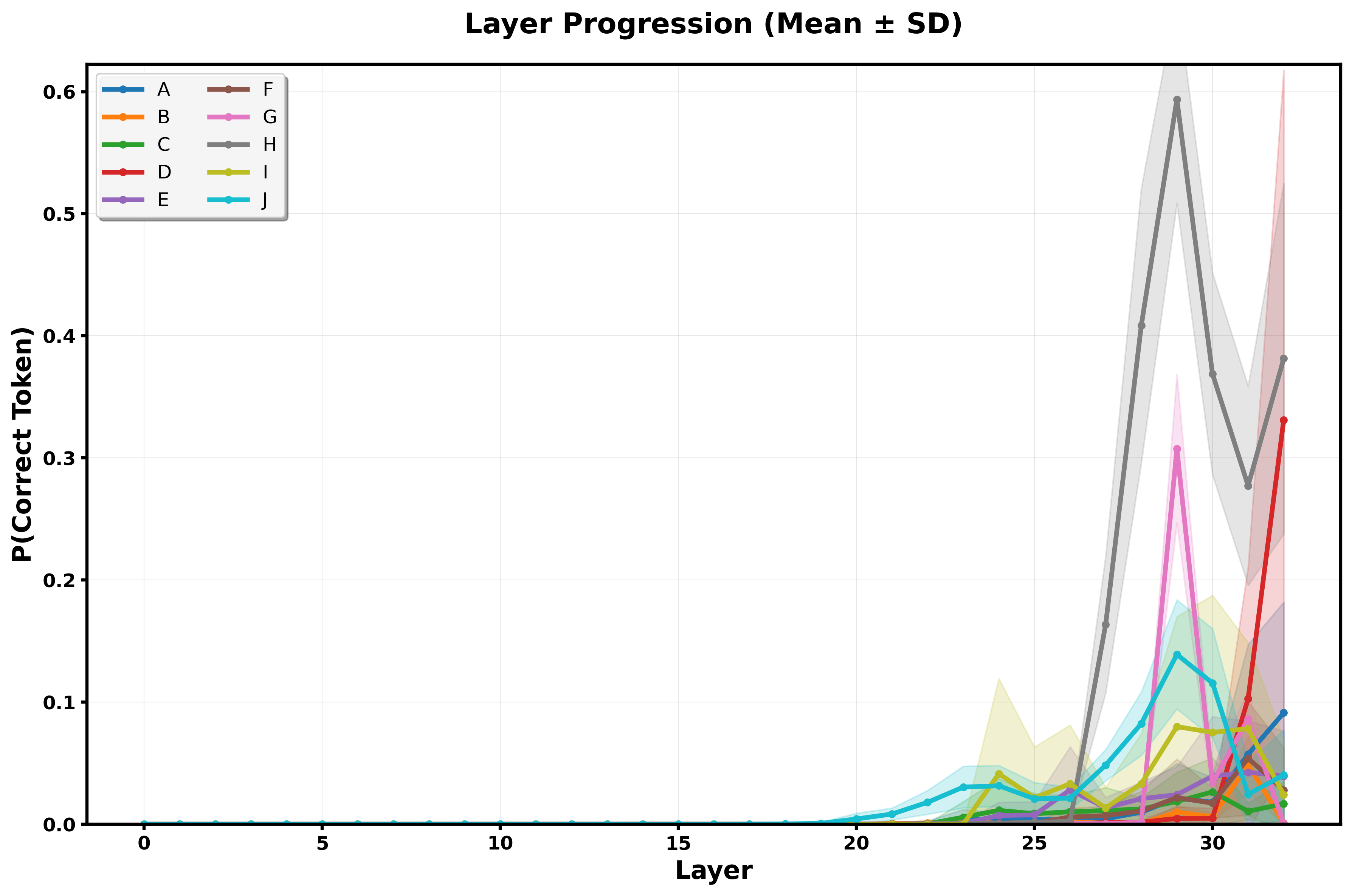}
    \caption{(Top) Relative contribution of MLP compared to multi-head attention (Bottom) Progression of probability of predicting the correct token for each transformation over all layers}
    \label{fig:mlp_vs_attention_layers}
\end{figure}

\textbf{Results.} We restrict our analysis to single-level transformations to avoid confounding effects from multiple simultaneous syntactic operations, allowing us to clearly measure the contributions of attention heads and MLP layers to individual transformation types. 

\textit{(1) Late-Layer Concentration.} Our analysis reveals that for any given transformation type, multiple attention heads contribute meaningfully to predicting the correct target token, with no single head acting as a bottleneck for the transformation process. However, the probability of predicting the token indicative of the transformation increases sharply in the final third of the network (layers 24-32), where 38\% of the contribution is concentrated. As shown in Figure~\ref{fig:mlp_vs_attention_layers} (bottom), the probability of predicting the correct token is near zero prior to layer 20. \looseness=-1 

\textit{(2) Component-Level Preferences.} Despite the distributed nature of processing, clear systematic biases emerge at the component level. MLP blocks consistently contribute more to syntactic transformation success than attention mechanisms, accounting for 65\% of total causal contribution versus 35\% for attention. In Figure~\ref{fig:mlp_vs_attention_layers}, the proportion of total contribution coming from the MLP is computed by summing over every layer. 

\textit{(3) Representation Quality.}
Complementing our causal analysis, linear discriminant analysis of intermediate representations shows near-perfect separation between transformation types, indicating that the model forms distinct internal representations for different syntactic operations, consistent with Section \ref{subsec:emergence}. This can additionally be understood as evidence for the linear representation hypothesis~\citep{park2024lrh}. In Figure~\ref{fig:projection_heatmap}, we show the projection values of the separating vector between transformation $\mathcal{A}$ and all other transformations. The average projection monotonically increases throughout the layers.

\begin{figure}[t]
    \centering
    \includegraphics[width=\columnwidth]{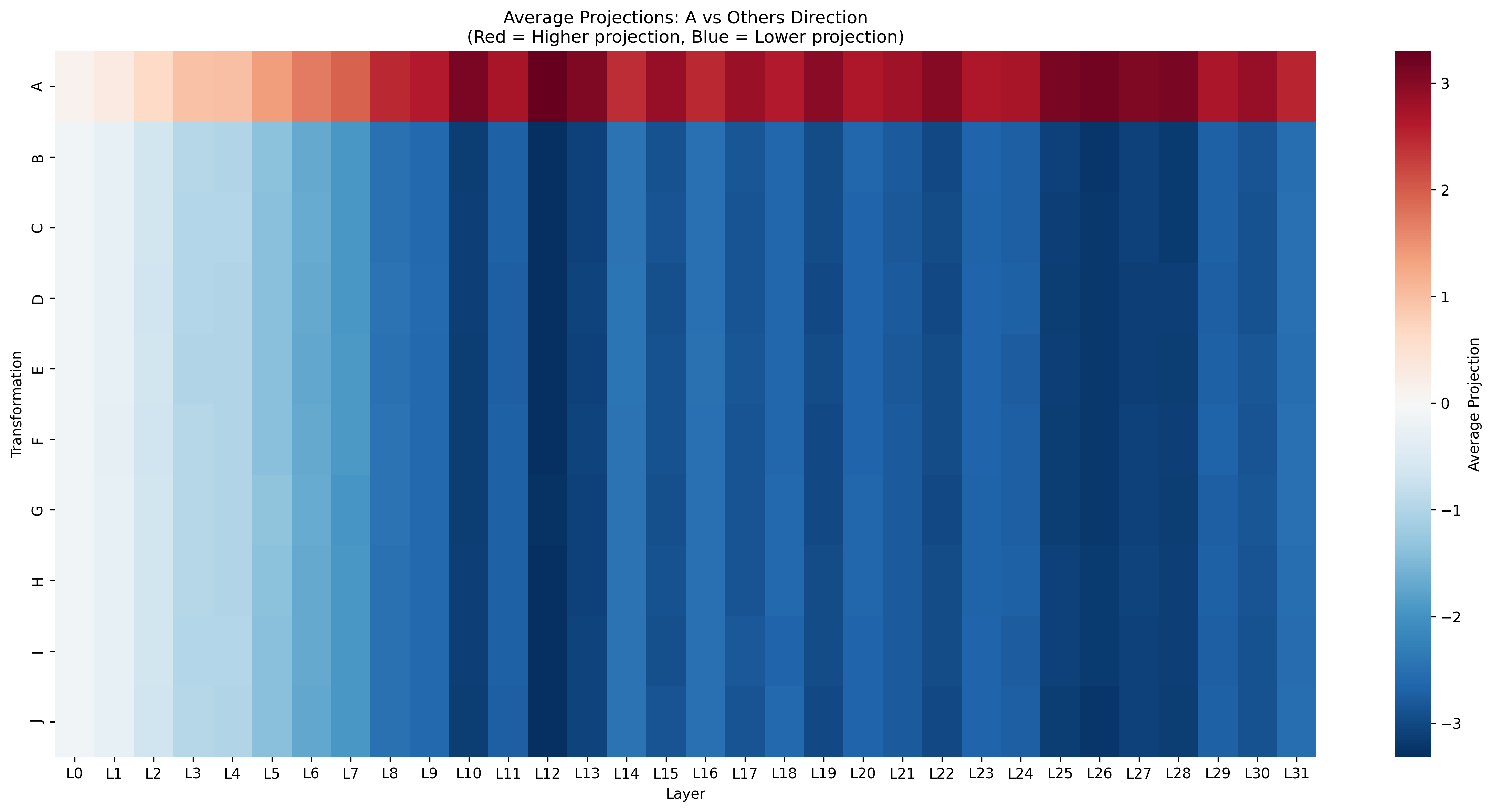}
    \caption{Heatmap showing average projections of intermediate representations onto the transformation A vs. others direction across all 32 layers. Each row represents a different transformation type (A-J), with transformation A contrasted against all other transformations. Red regions indicate higher projection values (representations more similar to transformation A), while blue regions indicate lower projection values (representations more dissimilar to transformation A).}
    \label{fig:projection_heatmap}
\end{figure}
\section{Conclusion}

We discussed the emergence of learning syntactic transformations during fine-tuning. We observe that LLMs achieve near-perfect accuracy on single transformations, and $74.37\%$ to $82.74\%$ exact accuracy on double transformations when providing intermediate representations and LoRA fine-tuning on Llama3-8B. Our findings resonate with recent literature as both full parameter updates as well as OOD evaluation on double transformations substantially decrease performance. To do syntactic transformations, models rely on the MLP sublayer and distinguish between different transformations at the final layers. We also provide evidence that the separation between concepts emerges as finetuning progresses and that this separation is linearly encoded in the model's residual stream. We finally introduce a natural language dataset on a broad variety of linguistic structural transformations and thereby show that LLMs can learn structural information during fine-tuning.

\section*{Limitations}
Our analysis focuses only on English transformations, limiting generalizability to the cross-linguistic patterns predicted by transformational grammar theory. The single-level transformation approach, while enabling clean causal attribution, may conflate multiple sub-operations within each syntactic transformation that could be attributed to distinct network components if decomposed further. Additionally, our findings on an $8$B model may not generalize to larger scales, though scaling would require more complex datasets to avoid memorization. Future work could address these through cross-linguistic studies, finer-grained operation decomposition, and evaluation on larger models with appropriately challenging transformation tasks. \looseness=-1

\section*{Acknowledgments}
M.C.C and M.M acknowledges Florent Draye for his insights and guidance on mechanistic interpretability and Hsiao-Ru Pan for feedback on the first draft. S.G. acknowledges the constructive discussions with Andreas Opedal and Ryan Cotterall on the manuscript, particularly regarding insights into computational linguistics. 
\bibliography{references}

\clearpage
\newpage
\onecolumn
\appendix

\section{Appendix}
\label{sec:appendix}

\subsection{Formalizing the dataset}
\begin{proof}[Proof of Proposition~\ref{prop:twocases}]
    Note that $\bigcup_{k \in \{1,...,R + 1\}} \Omega_k \subseteq \mathcal{S}$ by definition of $(\Omega_r)_{r \in \{1,...,R + 1\}}$. Similarly, $\mathcal{S} \subseteq \bigcup_{k \in \{1,...,R + 1\}} \Omega_k$ as we \textit{partition} the space into $R + 1$ subsets. Thus, for $r \in \{1,...,R\}, $
    \begin{align*}
        F_r: \Omega_r \rightarrow \left( \bigcup_{r' \in \{1,...,R \}} \Omega_{r'} \right) \cup \Omega_{R+1} = \mathcal{S}.
    \end{align*}
\end{proof}

\subsection{Progression of Sentence Representations}
\paragraph{Transformation Representation Clustering.}
For each transformation type (e.g., extraposition, passivization, raising), we measure the difference between mean token embeddings of transformed and original sentences from the final layer of Pythia-410M. 
The t-SNE visualizations (Figure~\ref{fig:tsne_clustering_progression}) show how initially overlapping transformation representations become well-separated by mid-training (16K–64K steps), revealing the emergence of structured clusters for distinct syntactic operations. 
PCA plots (Figure~\ref{fig:pca_clustering_progression}) further capture this progression, highlighting how the model organizes syntactic variation along major linear components. 
Clusters of similar color correspond to transformations within the same broad category (e.g., passives, raising, movement), suggesting hierarchical organization of grammatical knowledge.

\paragraph{L2 Norm Progression.}
The L2 norms of transformation embedding differences track how syntactic representations evolve during training. 
Related transformations—such as passive (\texttt{np\_passive\_1--3}), raising (\texttt{np\_raising\_1--3}), and movement (\texttt{i\_movement\_1}, \texttt{v\_movement\_1--2})—exhibit aligned trajectories, indicating coherent internal grouping. 
Most norms remain stable early on but rise sharply in later stages, paralleling performance gains. 
Across datasets (WikiText, PALOMA WikiText-103, PALOMA Penn Treebank), lower perplexity correlates with higher transformation norms, suggesting that improved language modeling coincides with more structured syntactic representations.

\subsection{Comprehensive Data Summary}
Our syntactic transformation dataset comprises 22,981 sentence pairs spanning 381,256 words (474,907 tokens). The dataset includes ten single-level transformations covering major syntactic operations: extraposition, movement (I-movement, V-movement), passivization (three variants), and raising constructions (three variants), totaling 19,100 examples. Additionally, we include eight nested transformation sequences that combine two operations (e.g., NP Raising + Extraposition, Passive + I Movement), contributing 3,881 examples. The nested transformations test compositional syntactic understanding by requiring models to process multiple sequential grammatical operations. Transformation counts range from 1,421 to 2,000 for single transformations and 466 to 499 for nested transformations.

\subsection{Detailed Transformation Descriptions}
\paragraph{Single Transformations}
Single transformations represent fundamental operations that systematically modify sentence structure. These include movement operations (Extraposition, I-Movement, V-Movement), voice alternations (NP Passive variants), and structural rearrangements (NP Raising variants). Each transformation follows specific patterns, moving constituents like noun phrases or auxiliary verbs. Examples and the patterns are shown in Table~\ref{tab:single_transformations_examples}.

\paragraph{Nested Transformations}
Nested transformations combine two compatible single-level transformations in sequence, creating complex derivations that test multi-step grammatical operations. By requiring sequential transformations, these examples provide challenging evaluations of syntactic competence and reveal how well models track structural dependencies across derivational steps. The patterns are the same as in Table~\ref{tab:single_transformations_examples} and examples are provided in Table~\ref{tab:nested_transformations_examples}.

\begin{figure*}[t]
    \centering
    \begin{subfigure}{.5\textwidth}
      \centering
      \includegraphics[width=\columnwidth]{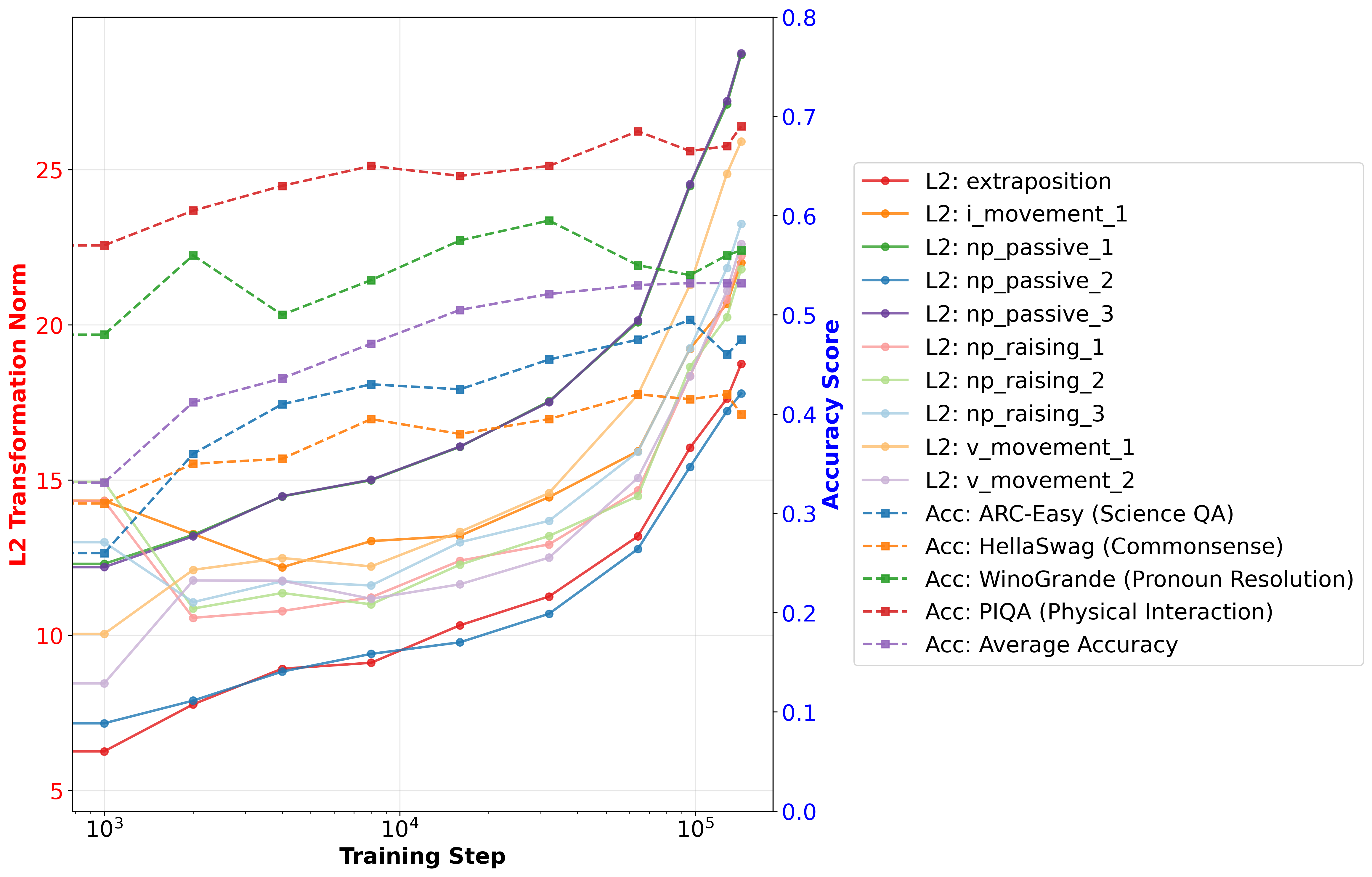}
    \end{subfigure}%
    \begin{subfigure}{.5\textwidth}
      \centering
      \includegraphics[width=\columnwidth]{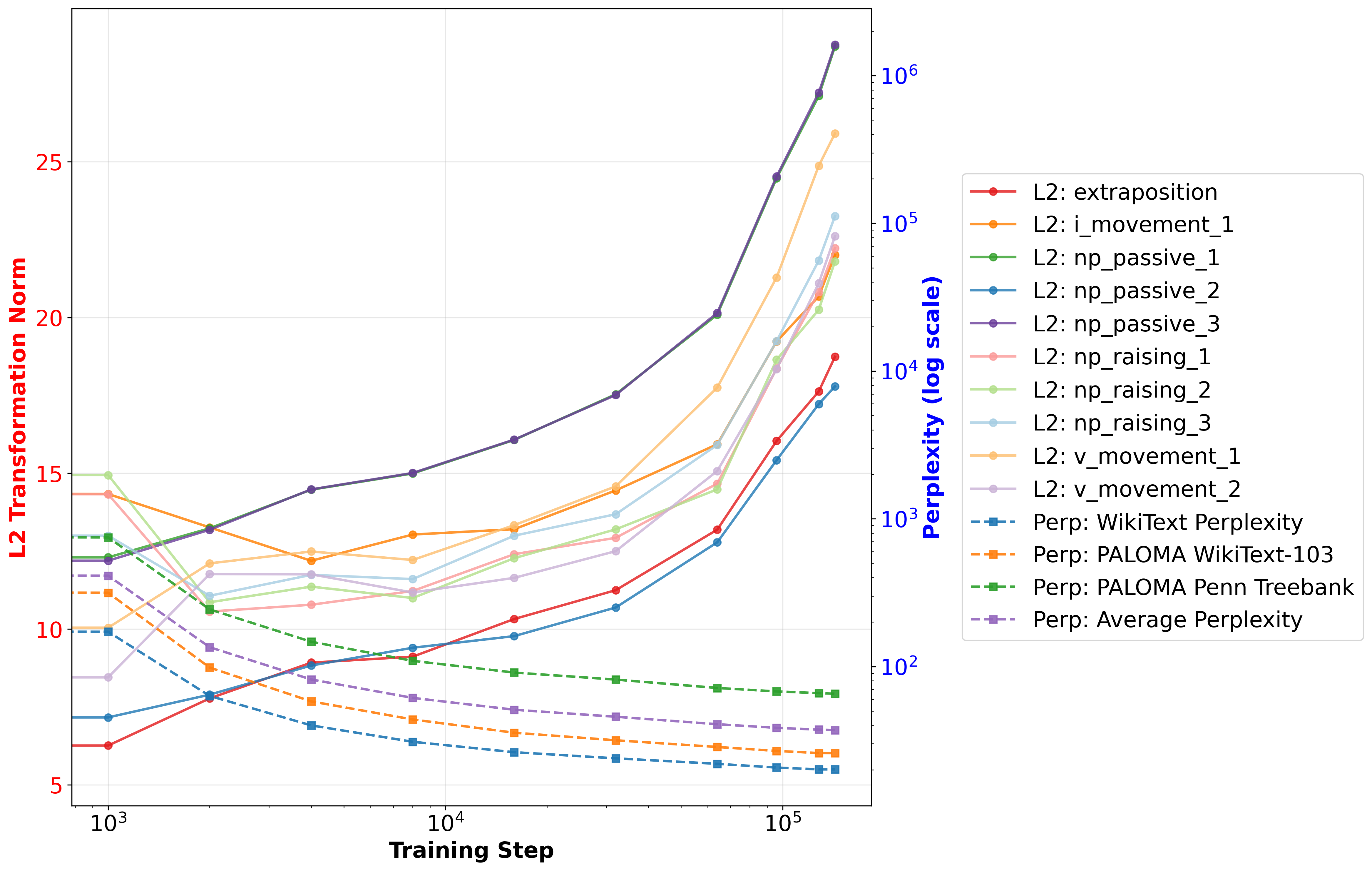}
    \end{subfigure}
    \caption{(Left) L2 norm progression compared to LM Harness Benchmarks. (Right) L2 norm progression compared to perplexity on WikiText-103 and Paloma Benchmarks}
    \label{fig:two-stacked}
\end{figure*}

\begin{table*}[t]
\centering
\caption{Dataset summary for single and nested transformations}
\begin{tabular}{|l|c|c|c|}
\hline
\textbf{Transformation} & \textbf{Count} & \textbf{Total Words} & \textbf{Total Tokens} \\
\hline
\multicolumn{4}{|c|}{\textbf{Single-level Transformations}} \\
\hline
Extraposition & 1999 & 28,497 & 35,118 \\
I Movement & 1917 & 23,077 & 28,148 \\
NP Passive 1 & 1785 & 24,016 & 29,908 \\
NP Passive 2 & 1996 & 26,154 & 32,356 \\
NP Passive 3 & 1421 & 19,222 & 23,967 \\
NP Raising 1 & 2000 & 32,536 & 40,683 \\
NP Raising 2 & 1982 & 31,784 & 39,762 \\
NP Raising 3 & 2000 & 41,090 & 51,906 \\
V Movement 1 & 2000 & 22,976 & 27,887 \\
V Movement 2 & 2000 & 26,055 & 31,888 \\
\hline
\textbf{Total Single} & \textbf{19,100} & \textbf{275,407} & \textbf{341,623} \\
\hline
\multicolumn{4}{|c|}{\textbf{Nested-level Transformations}} \\
\hline
NP Raising 1 + Extraposition & 466 & 16,752 & 21,135 \\
NP Raising 3 + I Movement & 472 & 11,952 & 14,946 \\
NP Passive 3 + I Movement & 467 & 13,851 & 17,622 \\
NP Passive 2 + I Movement & 495 & 10,500 & 13,327 \\
NP Raising 1 + NP Passive 1 & 488 & 13,282 & 16,701 \\
V Movement + I Movement & 498 & 9,387 & 11,422 \\
Extraposition + NP Passive 1 & 496 & 13,461 & 17,010 \\
Np Raising 3 + NP Passive 3 & 499 & 16,664 & 21,121 \\
\hline
\textbf{Total Nested} & \textbf{3,881} & \textbf{105,849} & \textbf{133,284} \\
\hline
\textbf{Overall} & \textbf{22,981} & \textbf{381,256} & \textbf{474,907} \\
\hline
\end{tabular}
\label{fig:dataset_summary}
\end{table*}

\begin{figure*}[t]
\centering
\begin{subfigure}[t]{0.18\textwidth}
    \centering
    \includegraphics[width=\textwidth]{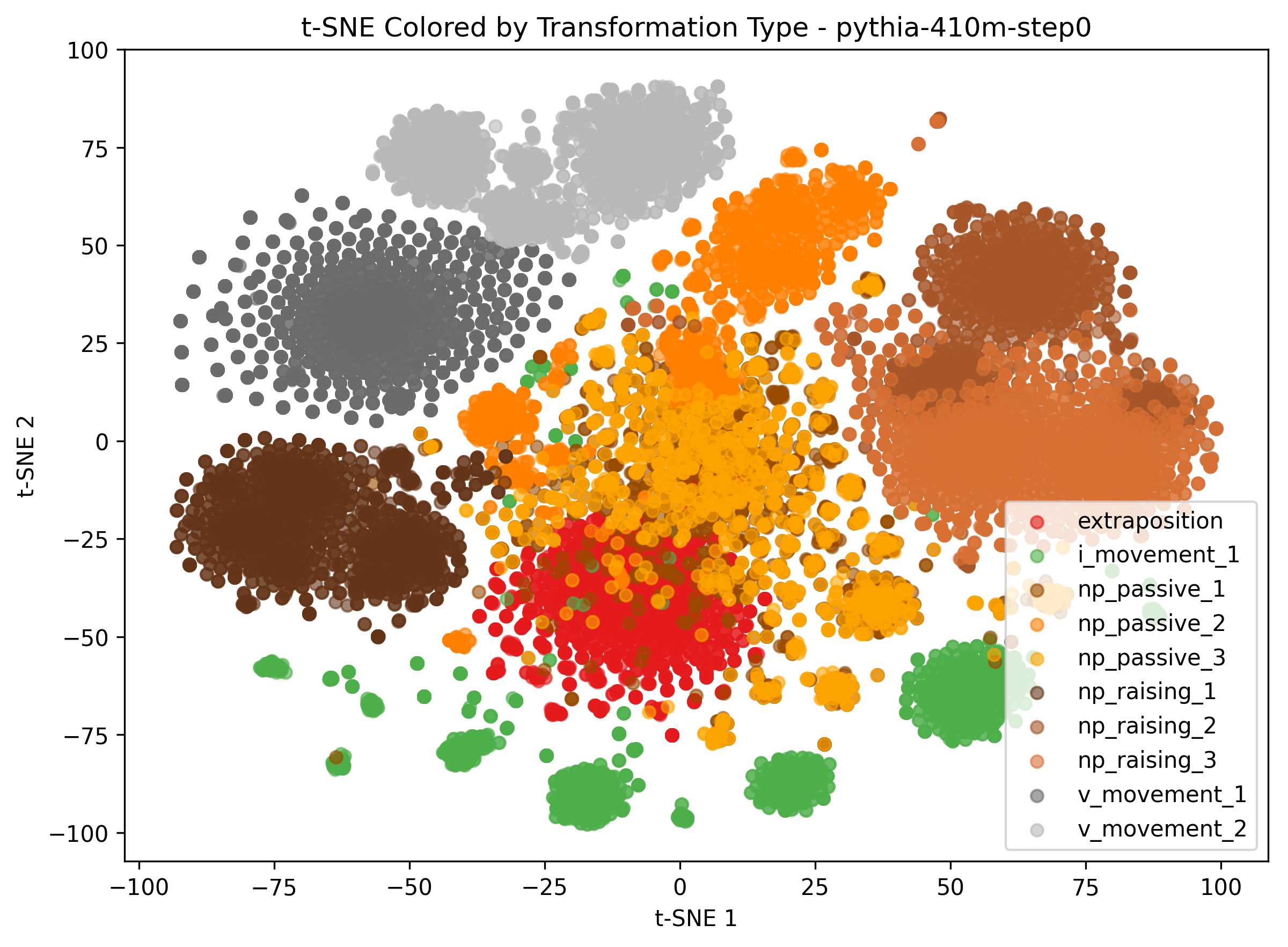}
    \caption{Step 0}
    \label{fig:tsne_step0}
\end{subfigure}
\hfill
\begin{subfigure}[t]{0.18\textwidth}
    \centering
    \includegraphics[width=\textwidth]{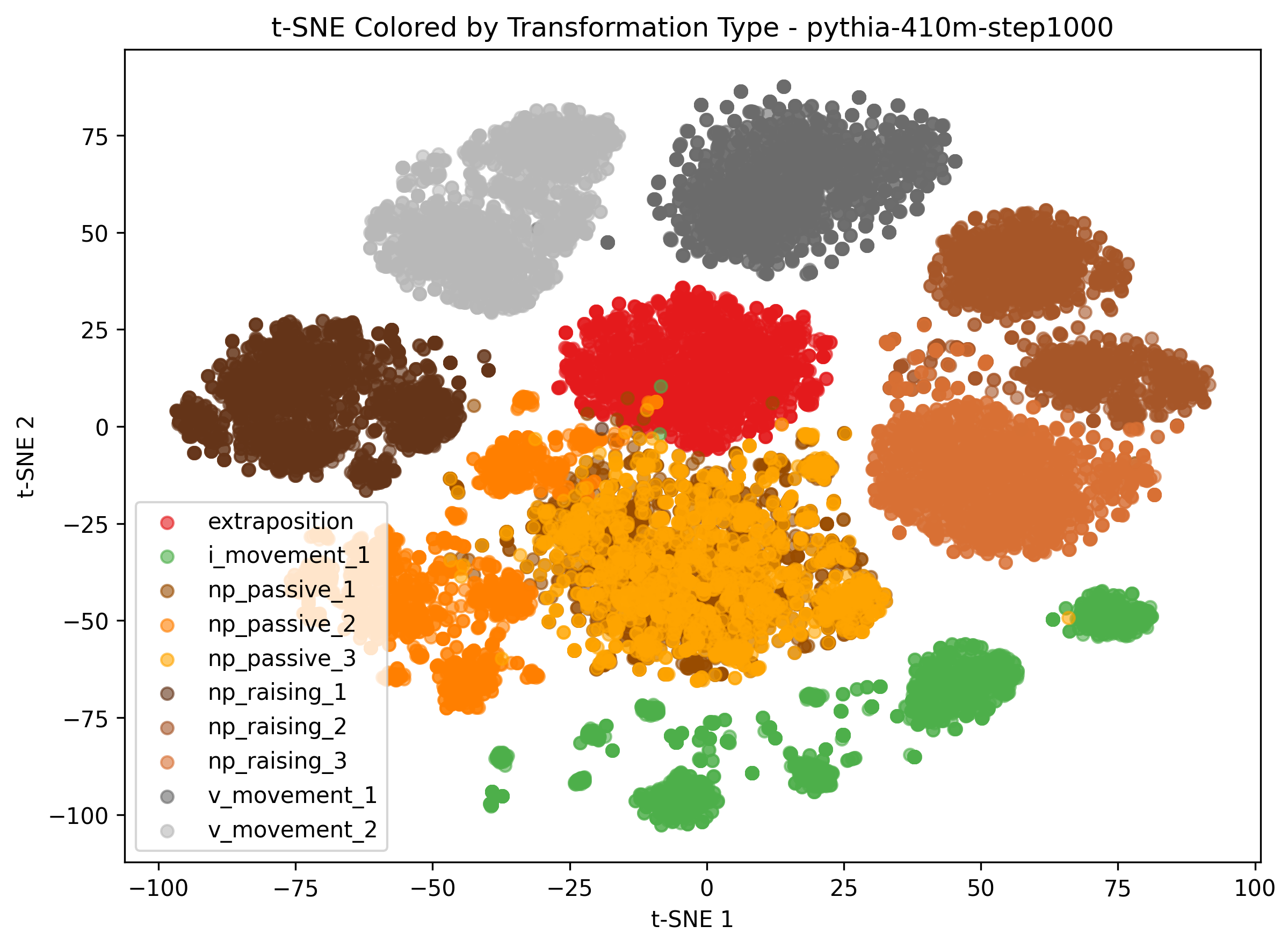}
    \caption{Step 1K}
    \label{fig:tsne_step1k}
\end{subfigure}
\hfill
\begin{subfigure}[t]{0.18\textwidth}
    \centering
    \includegraphics[width=\textwidth]{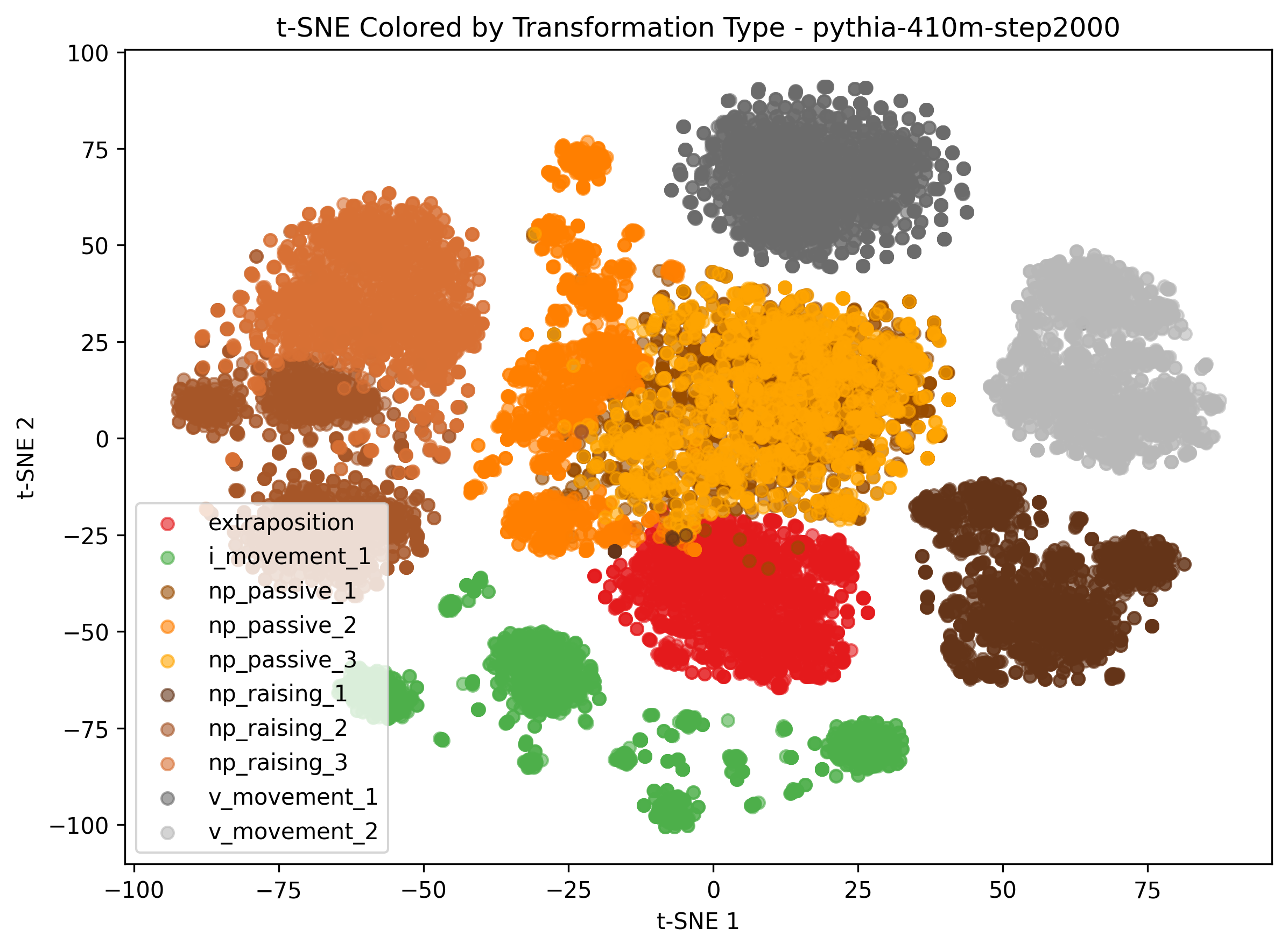}
    \caption{Step 2K}
    \label{fig:tsne_step2k}
\end{subfigure}
\hfill
\begin{subfigure}[t]{0.18\textwidth}
    \centering
    \includegraphics[width=\textwidth]{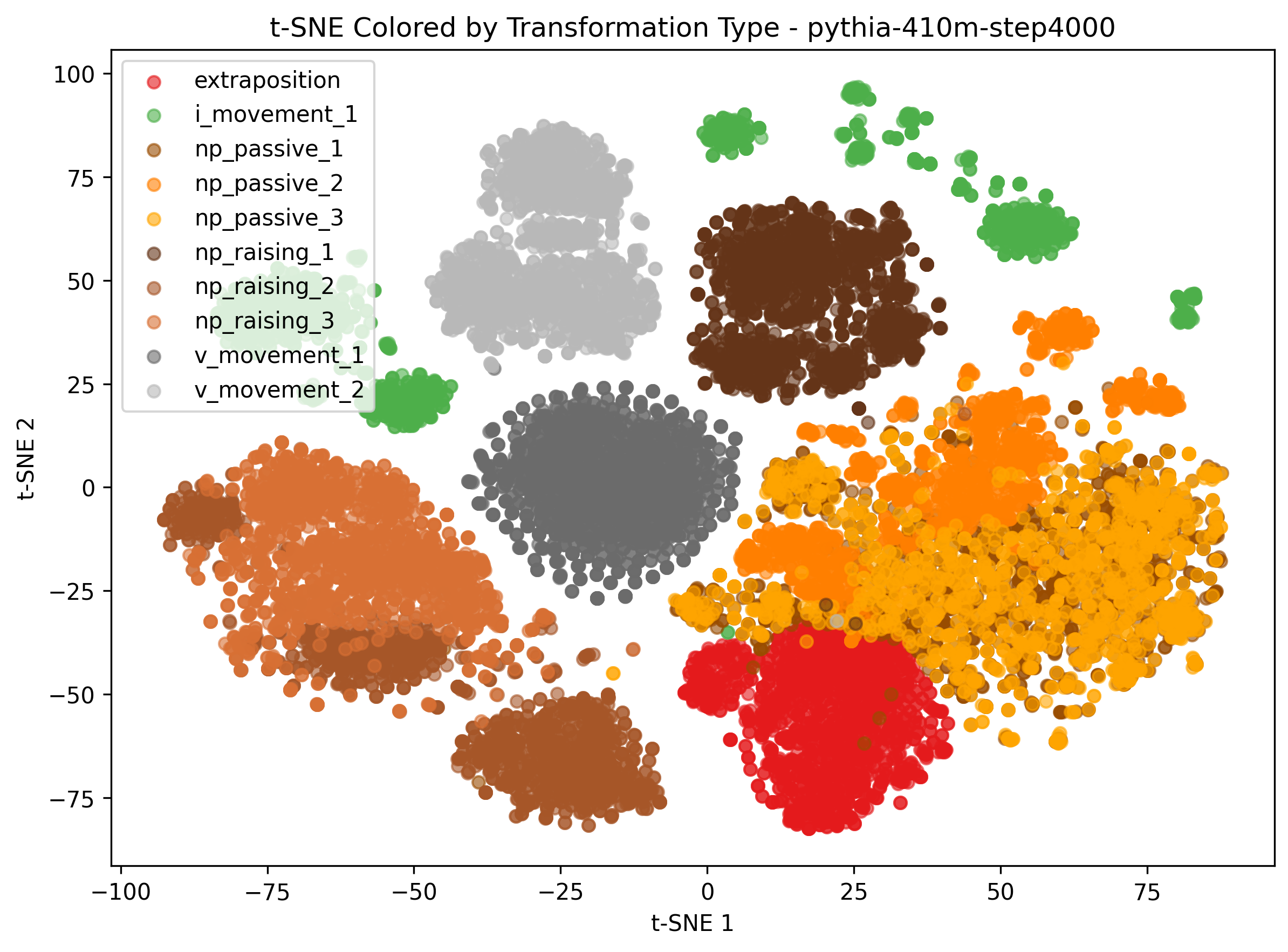}
    \caption{Step 4K}
    \label{fig:tsne_step4k}
\end{subfigure}
\hfill
\begin{subfigure}[t]{0.18\textwidth}
    \centering
    \includegraphics[width=\textwidth]{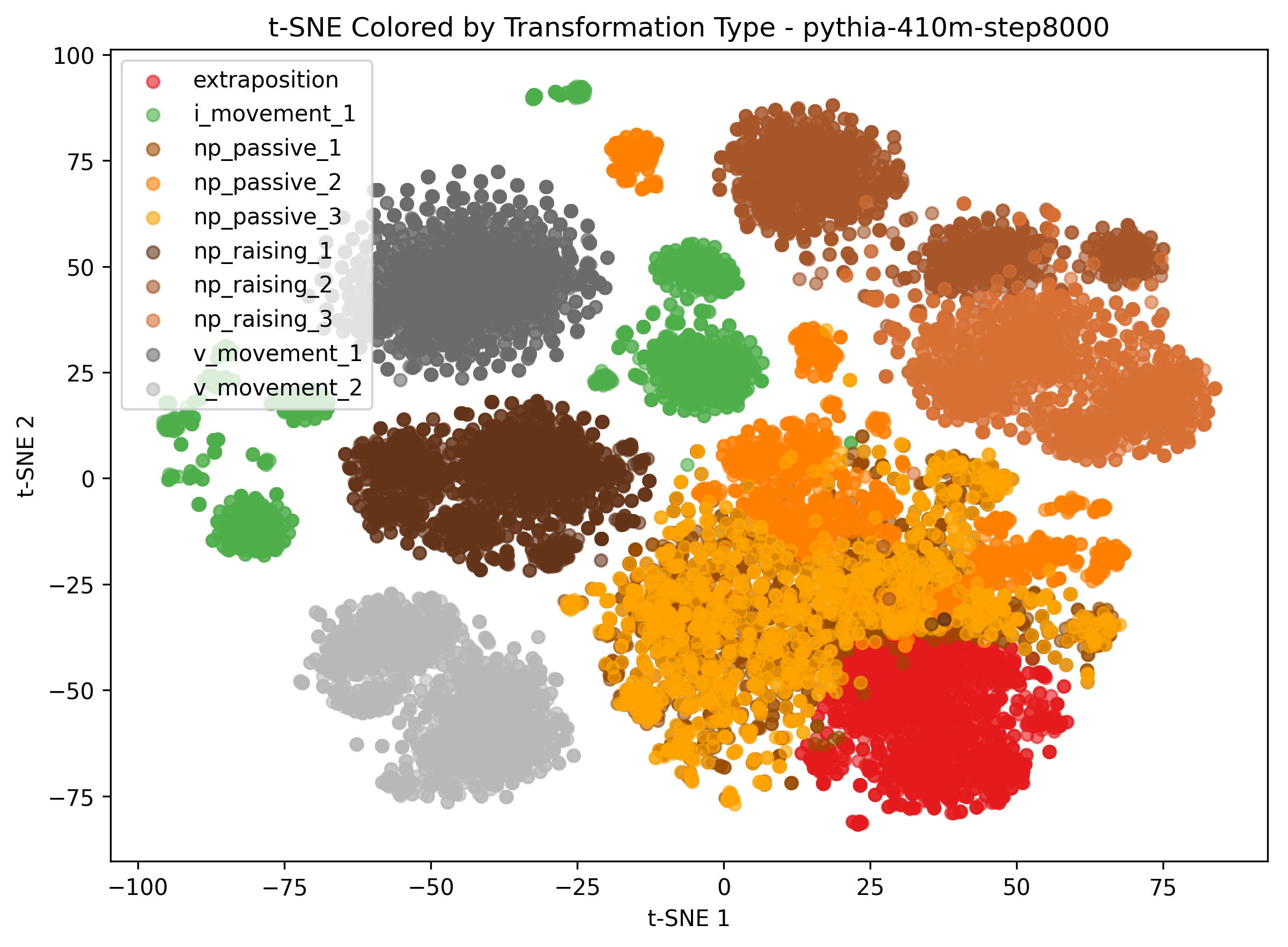}
    \caption{Step 8K}
    \label{fig:tsne_step8k}
\end{subfigure}

\vspace{0.5em}

\begin{subfigure}[t]{0.18\textwidth}
    \centering
    \includegraphics[width=\textwidth]{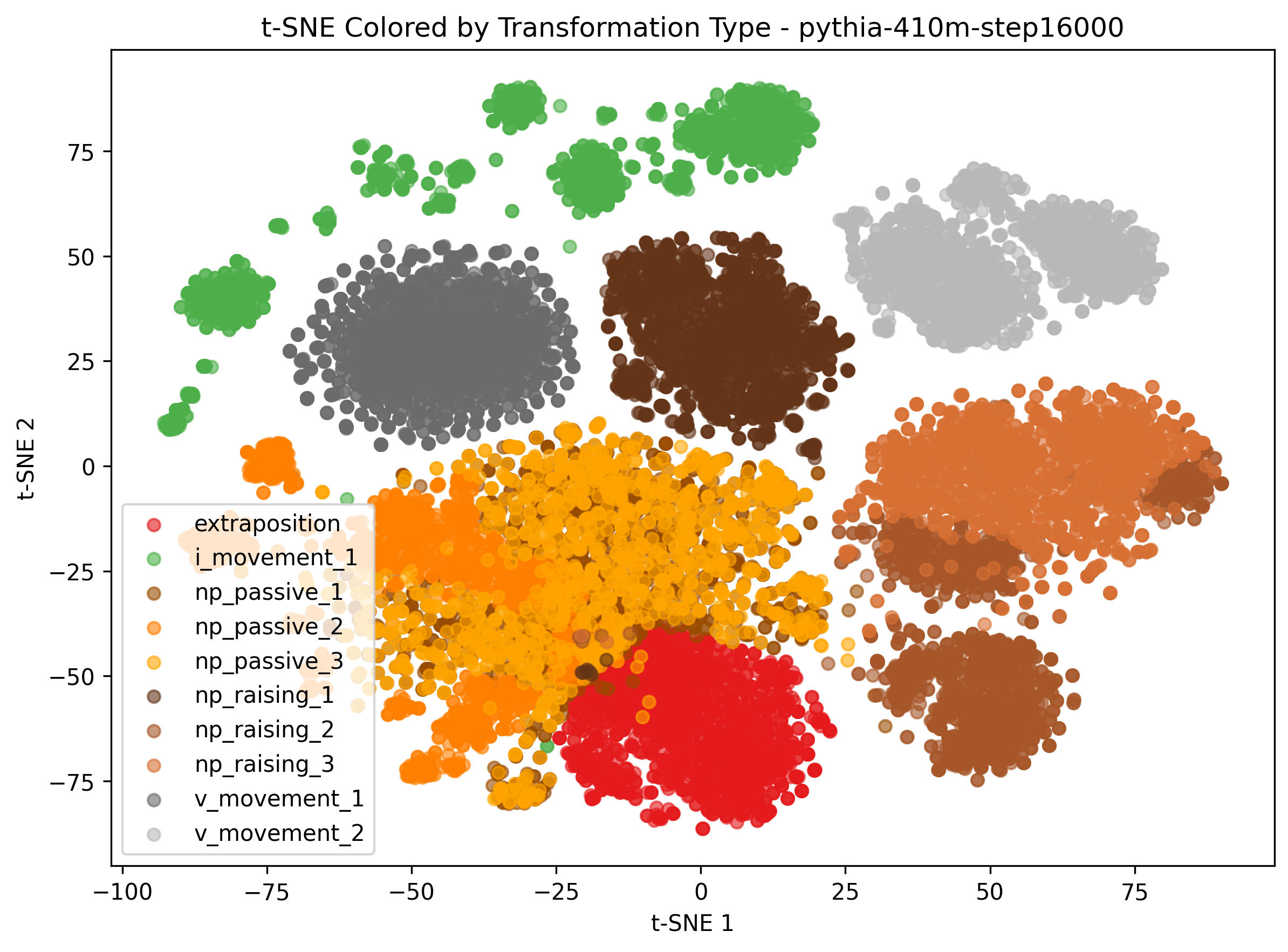}
    \caption{Step 16K}
    \label{fig:tsne_step16k}
\end{subfigure}
\hfill
\begin{subfigure}[t]{0.18\textwidth}
    \centering
    \includegraphics[width=\textwidth]{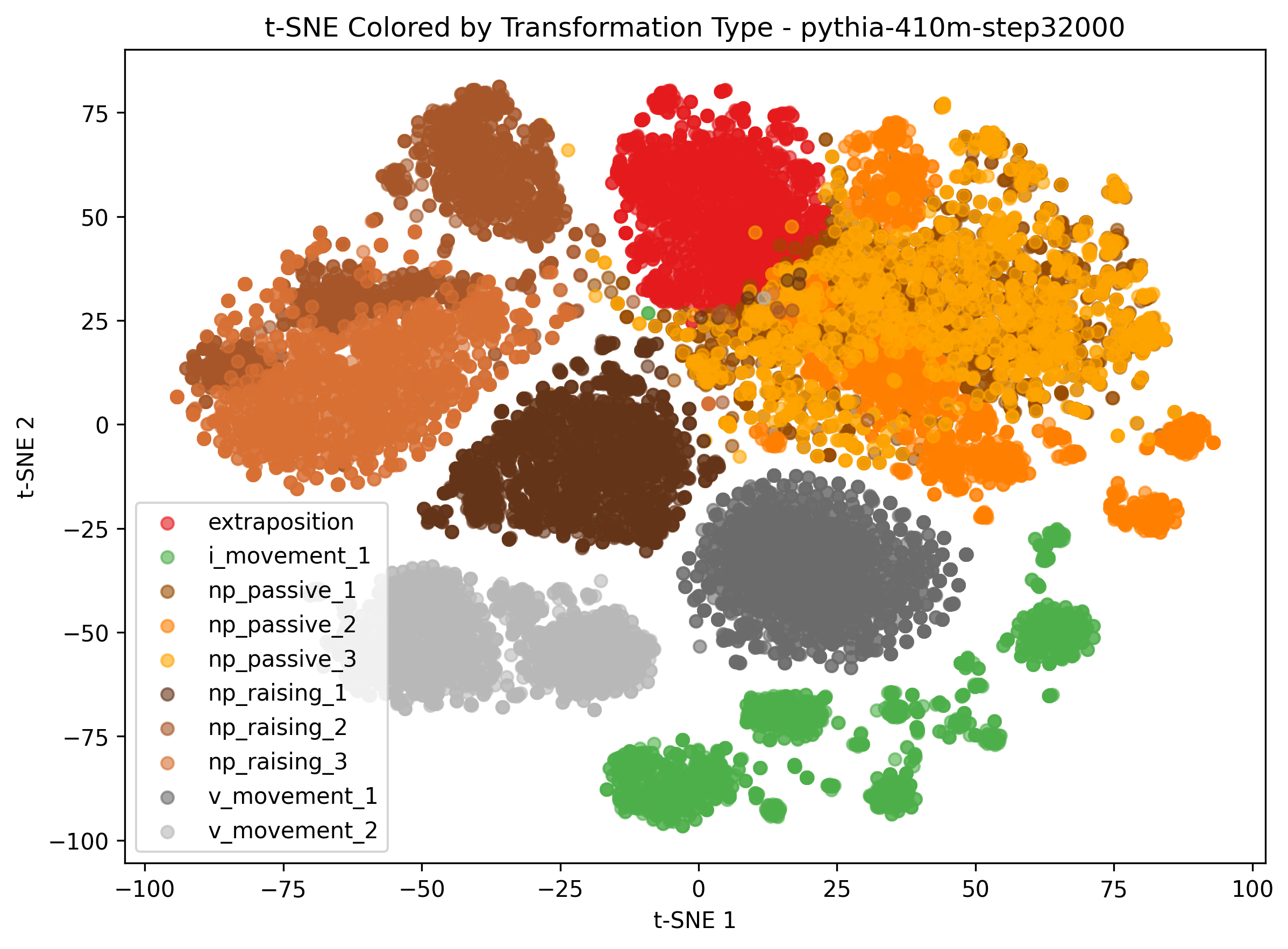}
    \caption{Step 32K}
    \label{fig:tsne_step32k}
\end{subfigure}
\hfill
\begin{subfigure}[t]{0.18\textwidth}
    \centering
    \includegraphics[width=\textwidth]{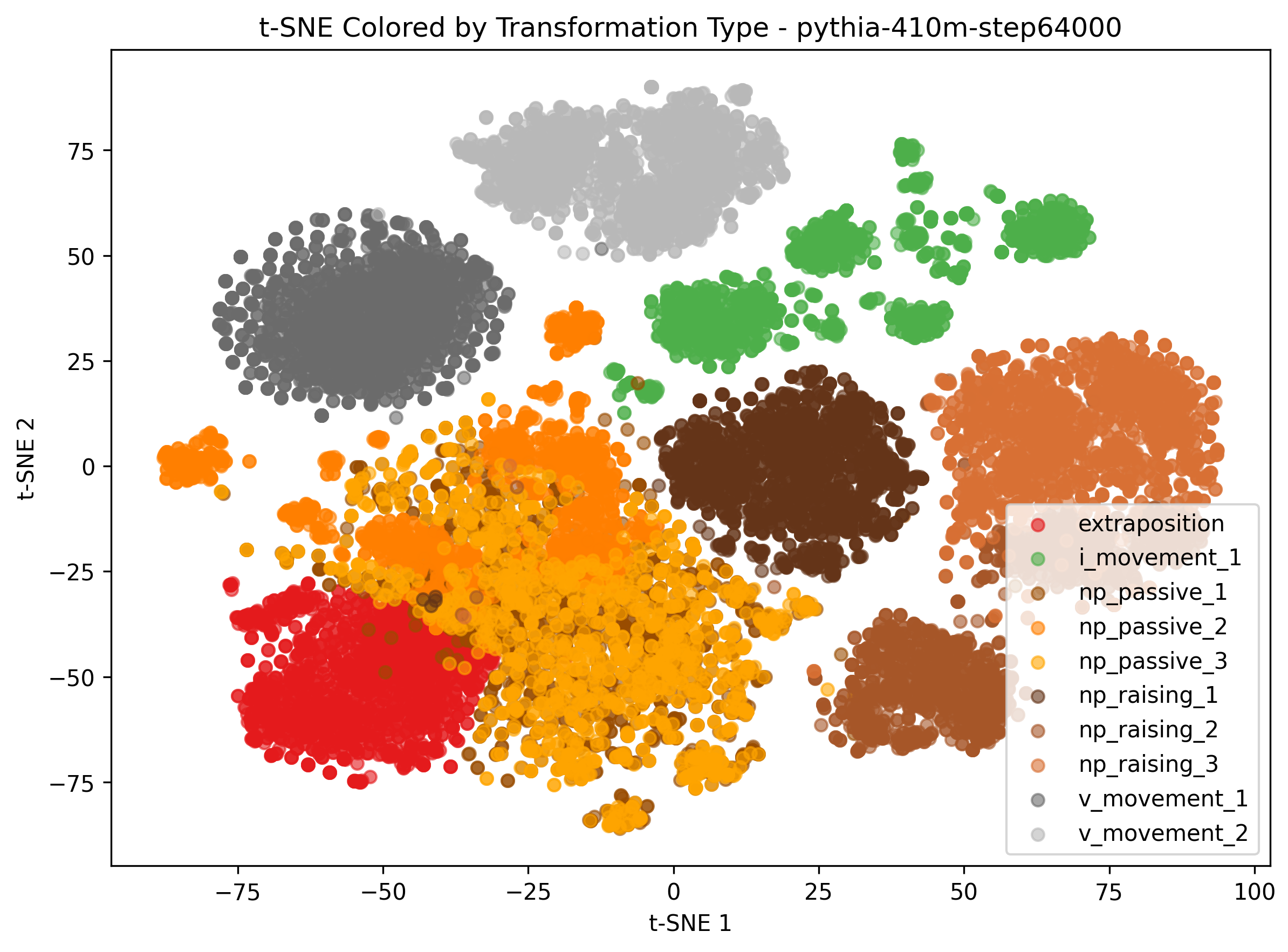}
    \caption{Step 64K}
    \label{fig:tsne_step64k}
\end{subfigure}
\hfill
\begin{subfigure}[t]{0.18\textwidth}
    \centering
    \includegraphics[width=\textwidth]{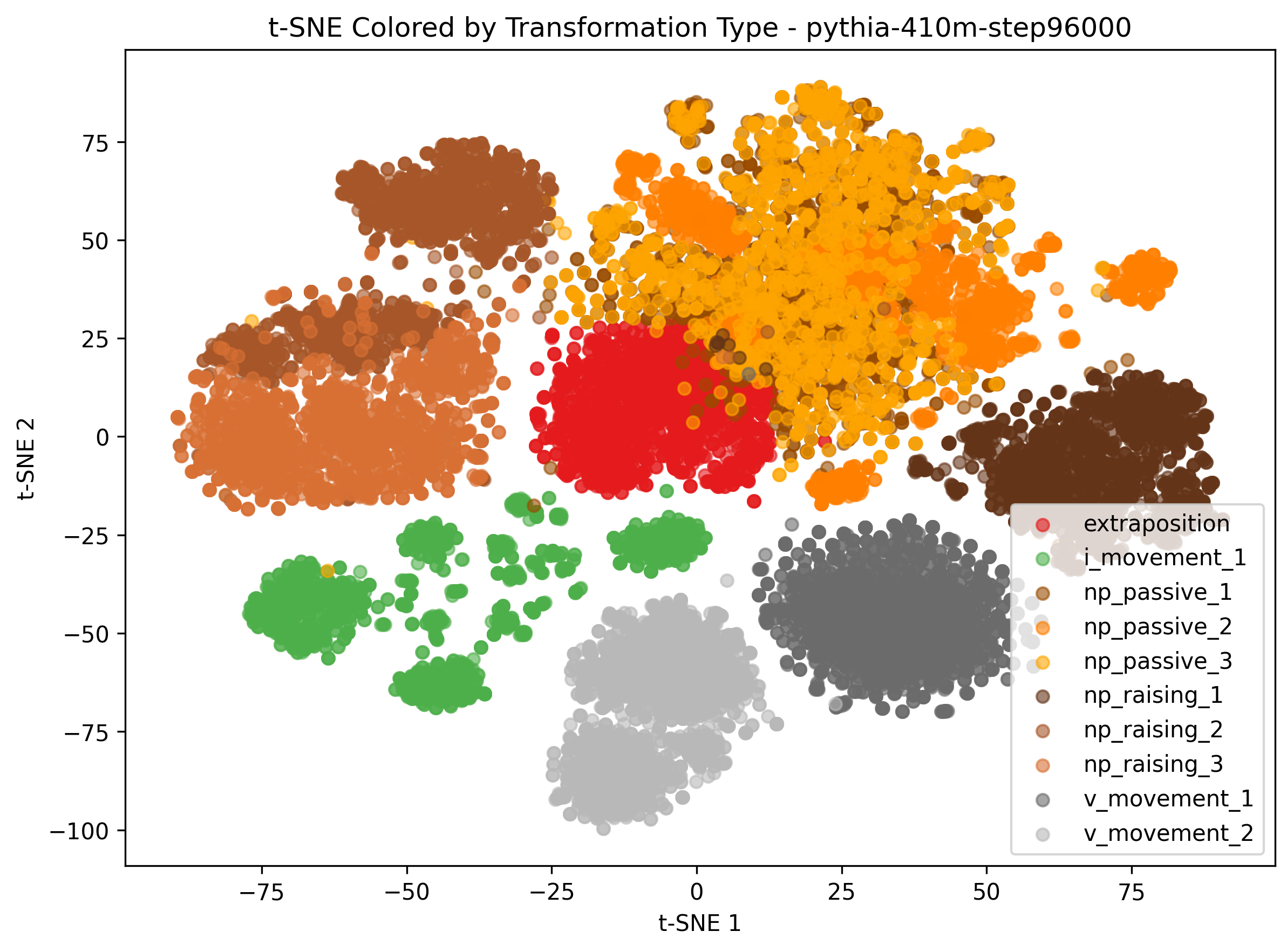}
    \caption{Step 96K}
    \label{fig:tsne_step96k}
\end{subfigure}
\hfill
\begin{subfigure}[t]{0.18\textwidth}
    \centering
    \includegraphics[width=\textwidth]{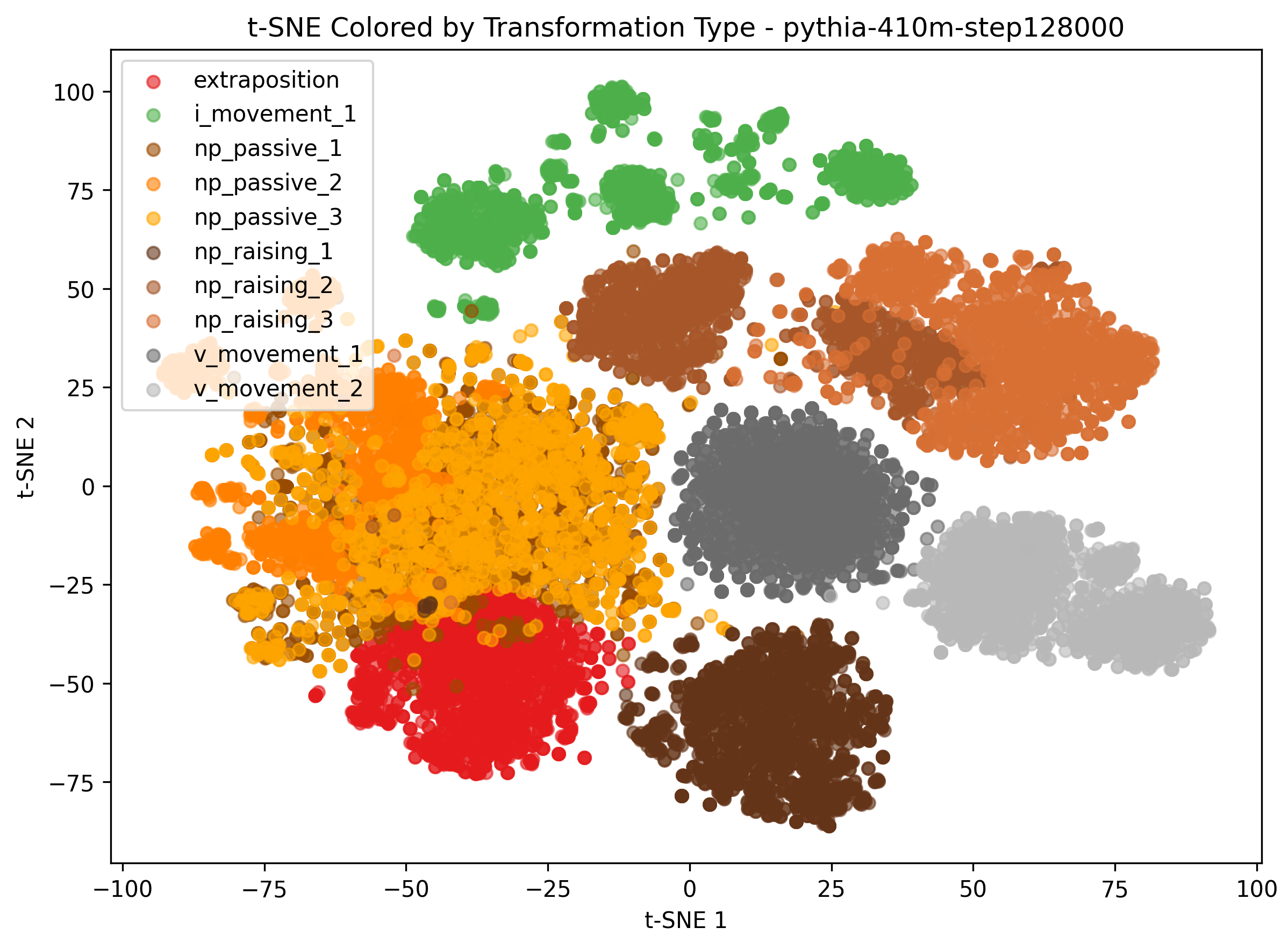}
    \caption{Step 128K}
    \label{fig:tsne_step128k}
\end{subfigure}

\vspace{0.5em}

\begin{subfigure}[t]{0.18\textwidth}
    \centering
    \includegraphics[width=\textwidth]{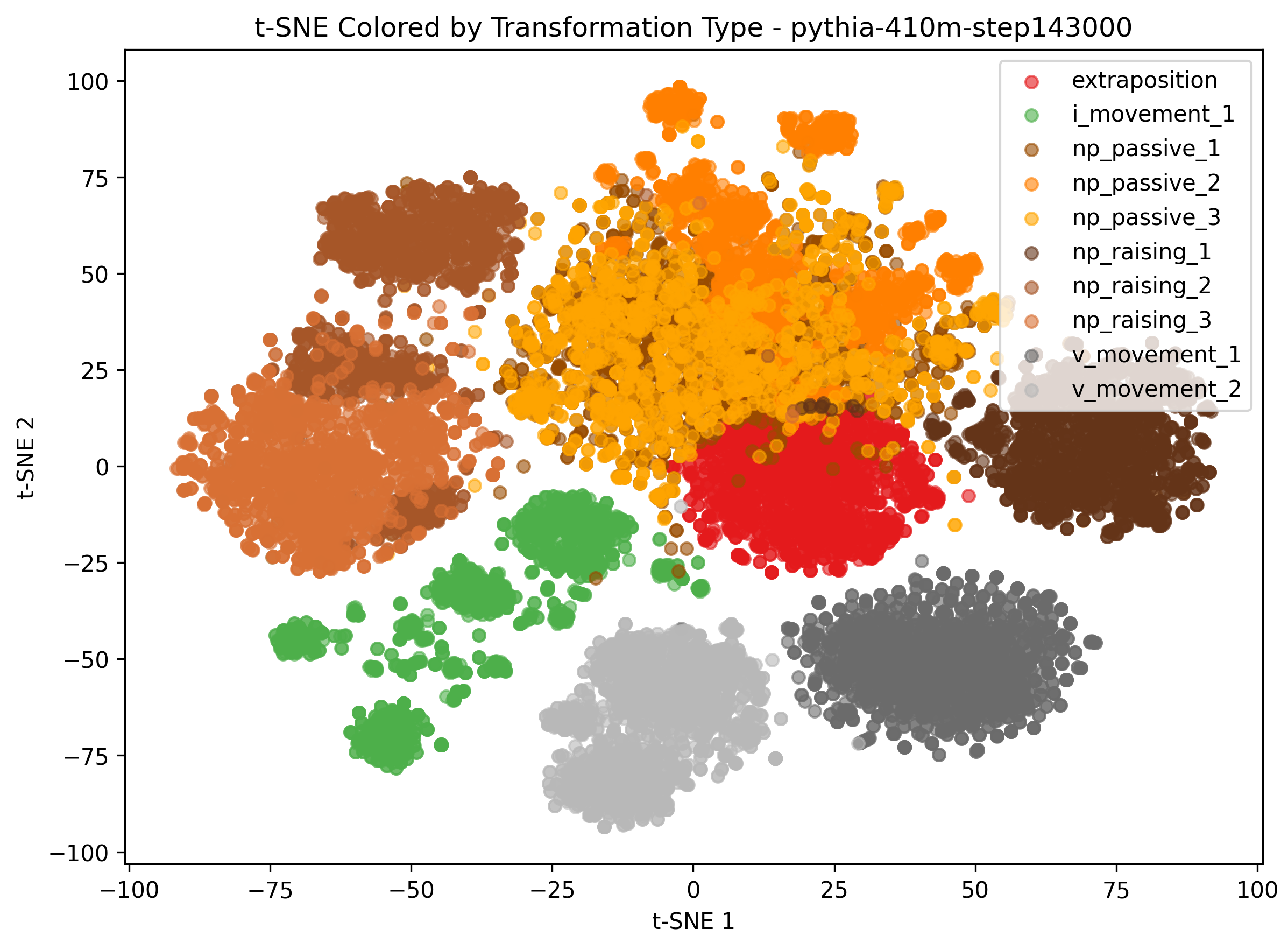}
    \caption{Step 143K}
    \label{fig:tsne_step143k}
\end{subfigure}

\caption{t-SNE visualization of syntactic transformation embedding differences across Pythia-410M training. Each subplot shows the 2D t-SNE projection of embedding difference vectors (transformed - original) for all transformation types at different training steps (0, 1K, 2K, 4K, 8K, 16K, 32K, 64K, 96K, 128K, 143K). Different colors represent different syntactic transformations.}
\label{fig:tsne_clustering_progression}
\end{figure*}

\begin{table*}[t]
\centering
\caption{Description and examples for single-level transformations}
\rowcolors{2}{gray!10}{white} 
\begin{tabularx}{\textwidth}{|l|X|X|X|}
\hline
\textbf{Transformation} & \textbf{Description} & \textbf{Pattern} & \textbf{Example} \\
\hline
\textbf{Extraposition} & Moves prepositional phrases from within noun phrases to sentence-final position & \texttt{[NP + PP][VP] → [NP][VP][PP]} & "The book on the table disappeared" → "The book disappeared on the table" \\
\textbf{I-Movement} & Moves auxiliaries/modals to sentence-initial position to form questions & \texttt{[NP][Aux/Modal][VP] → [Aux/Modal][NP][VP]?} & "She can swim" → "Can she swim?" \\
\textbf{NP Passive 1} & Standard active-to-passive with transitive verbs & \texttt{[NP subject][V][NP object] → [NP object][be + past participle][by NP subject]} & "The teacher graded the exams" → "The exams were graded by the teacher" \\
\textbf{NP Passive 2} & Small clause transformation with expanded structures & \texttt{[V][NP][small clause] → [V][that/to clause]} & "I consider him intelligent" → "I consider him to be intelligent" \\
\textbf{NP Passive 3} & Active-to-passive with clear subject-verb-object structures & \texttt{[NP subject][V][NP object] → [NP object][be + past participle][by NP subject]} & "The scientist discovered the formula" → "The formula was discovered by the scientist" \\
\textbf{NP Raising 1} & Expletive "it" structure to raised subject & \texttt{[It][verb][that[NP VP]] → [NP][verb][to VP]} & "It seems that John is happy" → "John seems to be happy" \\
\textbf{NP Raising 2} & Reverse raising: raised subject back to expletive structure & \texttt{[NP][verb][to VP] → [It][verb][that[NP VP]]} & "Mary seems to be tired" → "It seems that Mary is tired" \\
\textbf{NP Raising 3} & Complex raising with experiencer & \texttt{[NP\_1][verb][to NP\_2][to VP] → [It][verb][to NP\_2][that[NP\_1 VP]]} & "John seems to me to be honest" → "It seems to me that John is honest" \\
\textbf{V-Movement 1} & Integrates separated infinitive components & \texttt{[NP]; [VP infinitive] → [NP][VP finite]} & "The children; to play outside" → "The children play outside" \\
\textbf{V-Movement 2} & Integrates separated modal components & \texttt{[NP]; [modal]; [VP] → [NP][modal VP]} & "The students; can; solve the problem" → "The students can solve the problem" \\
\hline
\end{tabularx}
\label{tab:single_transformations_examples}
\end{table*}

\begin{figure*}[t]
\centering
\begin{subfigure}[t]{0.18\textwidth}
    \centering
    \includegraphics[width=\textwidth]{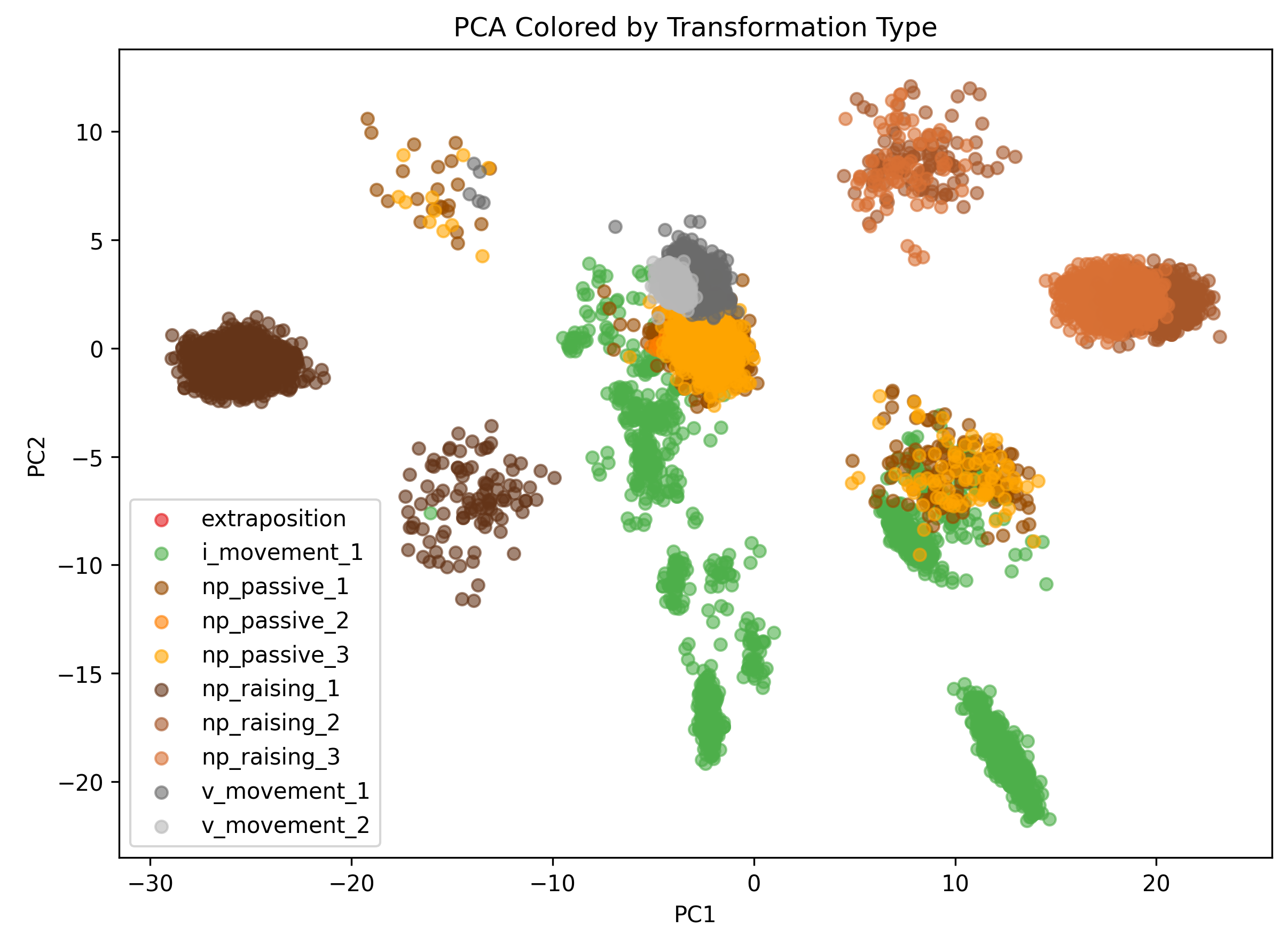}
    \caption{Step 0}
    \label{fig:clustering_step0}
\end{subfigure}
\hfill
\begin{subfigure}[t]{0.18\textwidth}
    \centering
    \includegraphics[width=\textwidth]{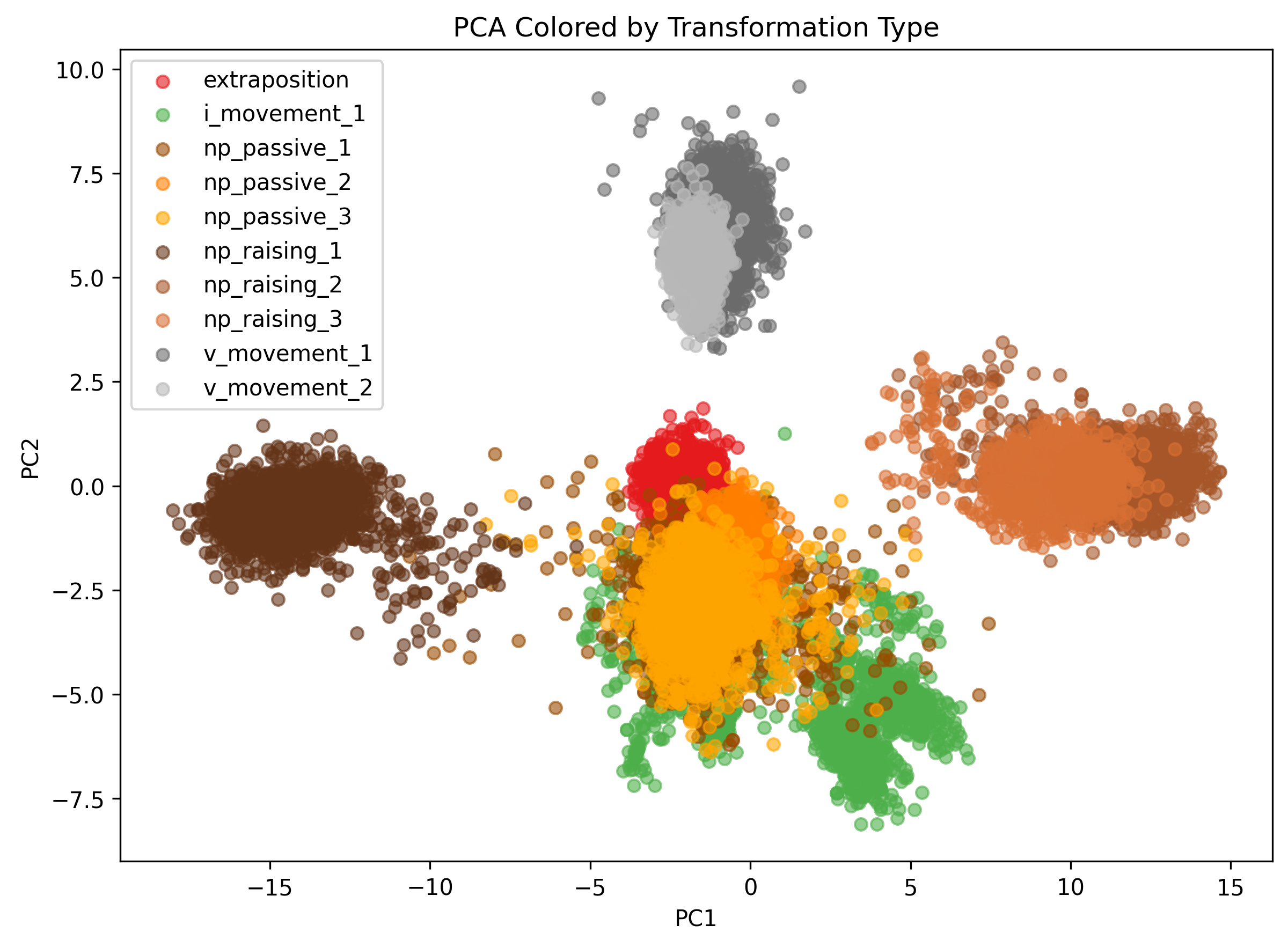}
    \caption{Step 1K}
    \label{fig:clustering_step1k}
\end{subfigure}
\hfill
\begin{subfigure}[t]{0.18\textwidth}
    \centering
    \includegraphics[width=\textwidth]{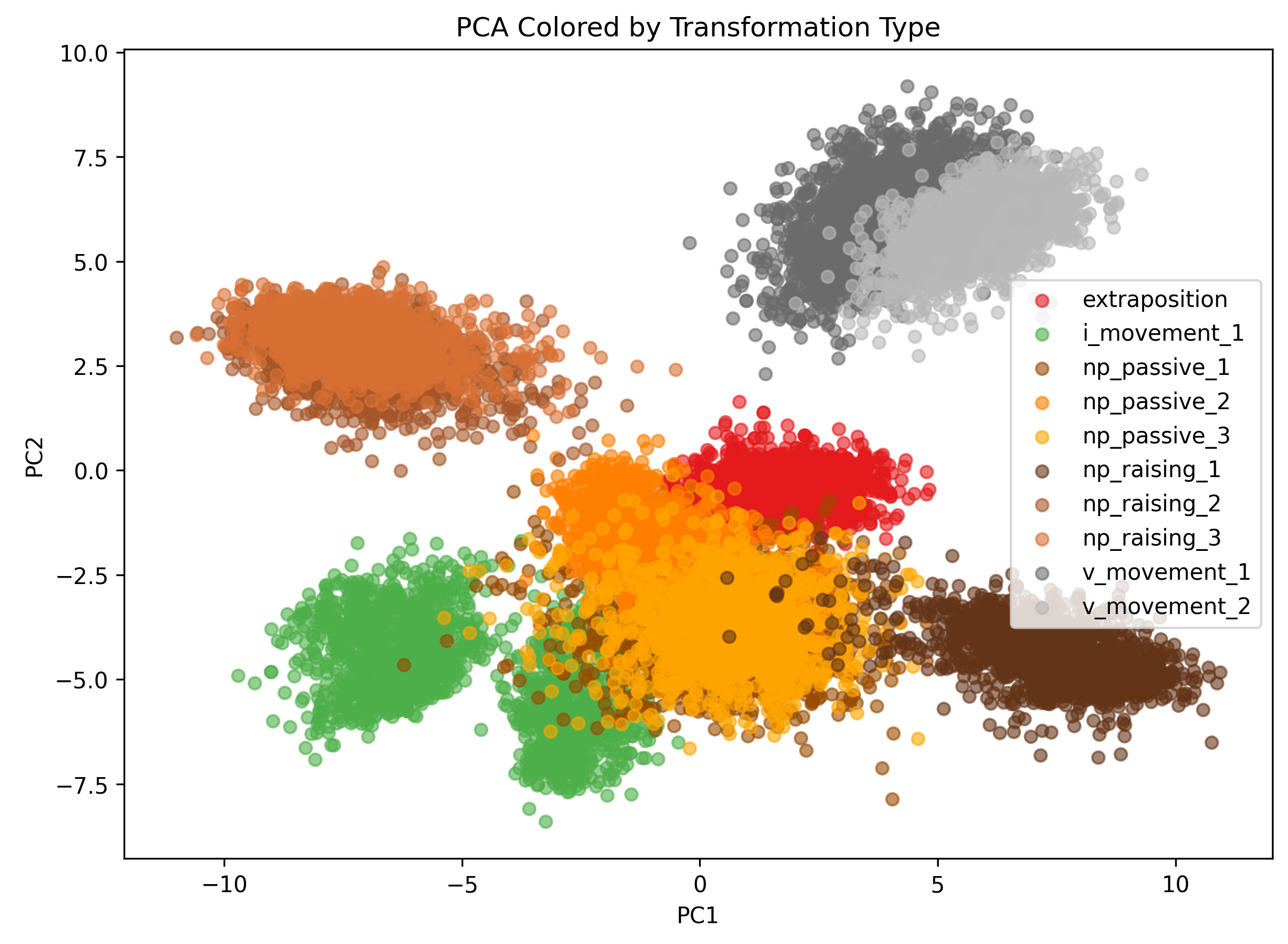}
    \caption{Step 2K}
    \label{fig:clustering_step2k}
\end{subfigure}
\hfill
\begin{subfigure}[t]{0.18\textwidth}
    \centering
    \includegraphics[width=\textwidth]{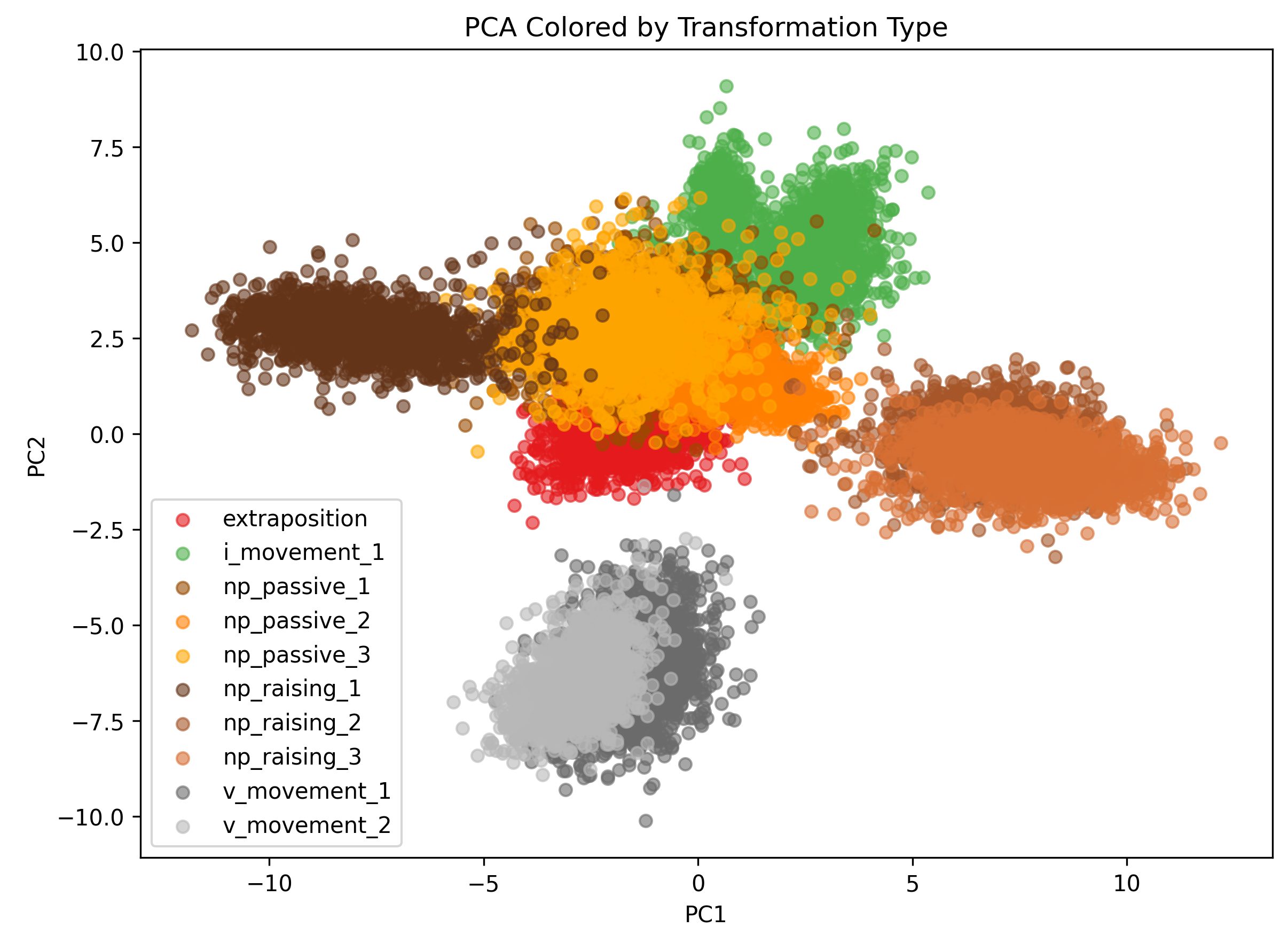}
    \caption{Step 4K}
    \label{fig:clustering_step4k}
\end{subfigure}
\hfill
\begin{subfigure}[t]{0.18\textwidth}
    \centering
    \includegraphics[width=\textwidth]{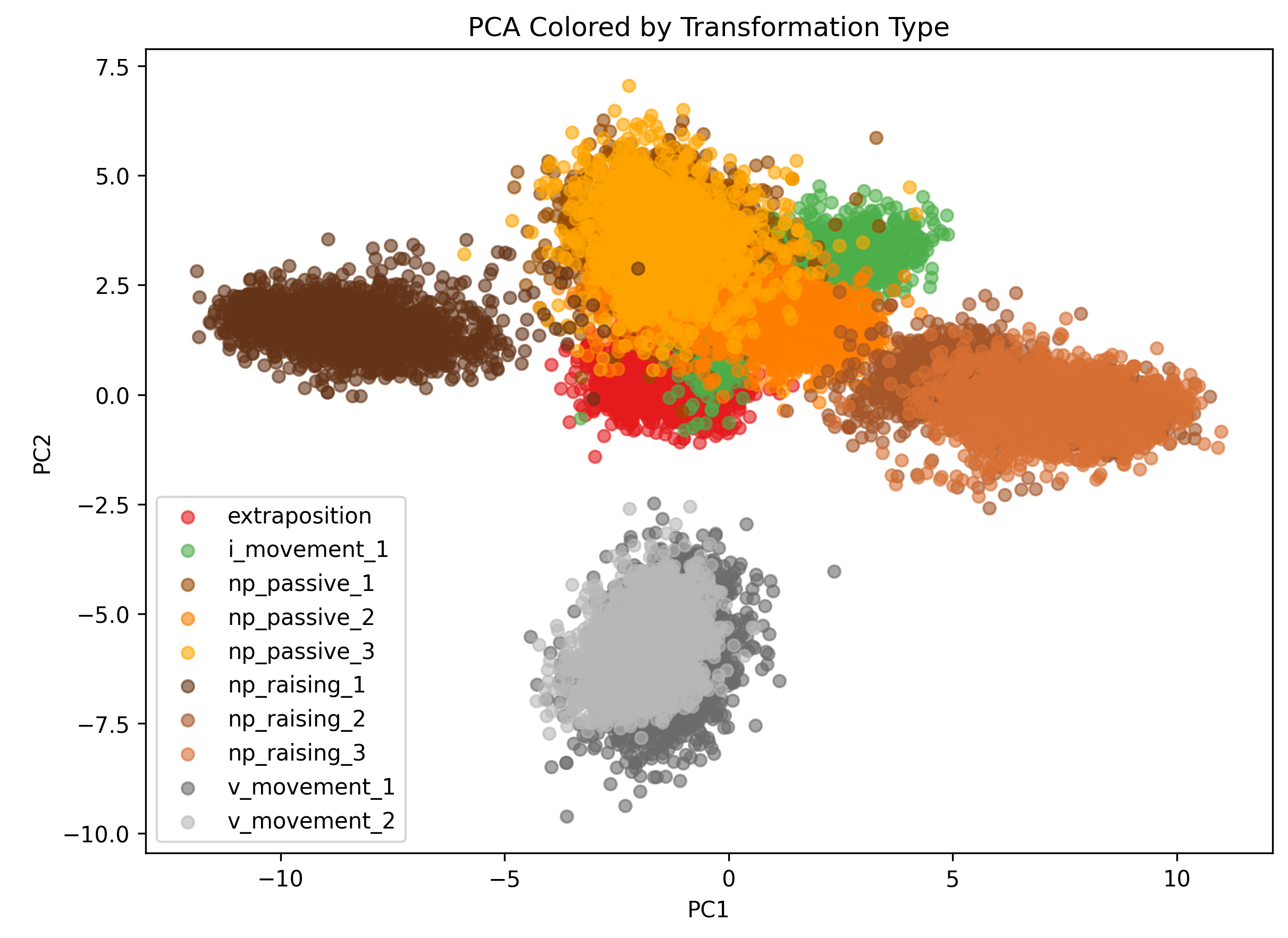}
    \caption{Step 8K}
    \label{fig:clustering_step8k}
\end{subfigure}

\vspace{0.5em}

\begin{subfigure}[t]{0.18\textwidth}
    \centering
    \includegraphics[width=\textwidth]{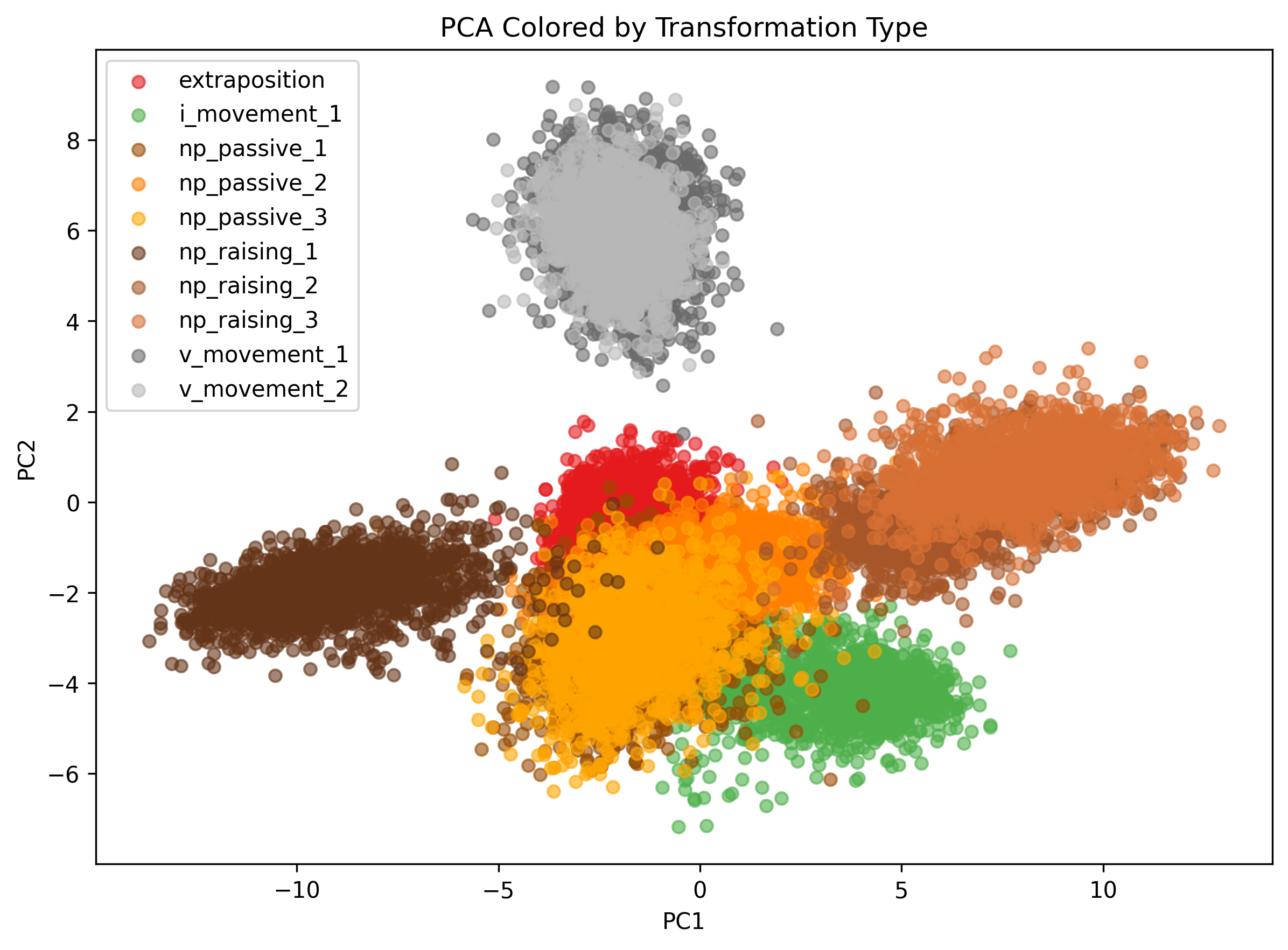}
    \caption{Step 16K}
    \label{fig:clustering_step16k}
\end{subfigure}
\hfill
\begin{subfigure}[t]{0.18\textwidth}
    \centering
    \includegraphics[width=\textwidth]{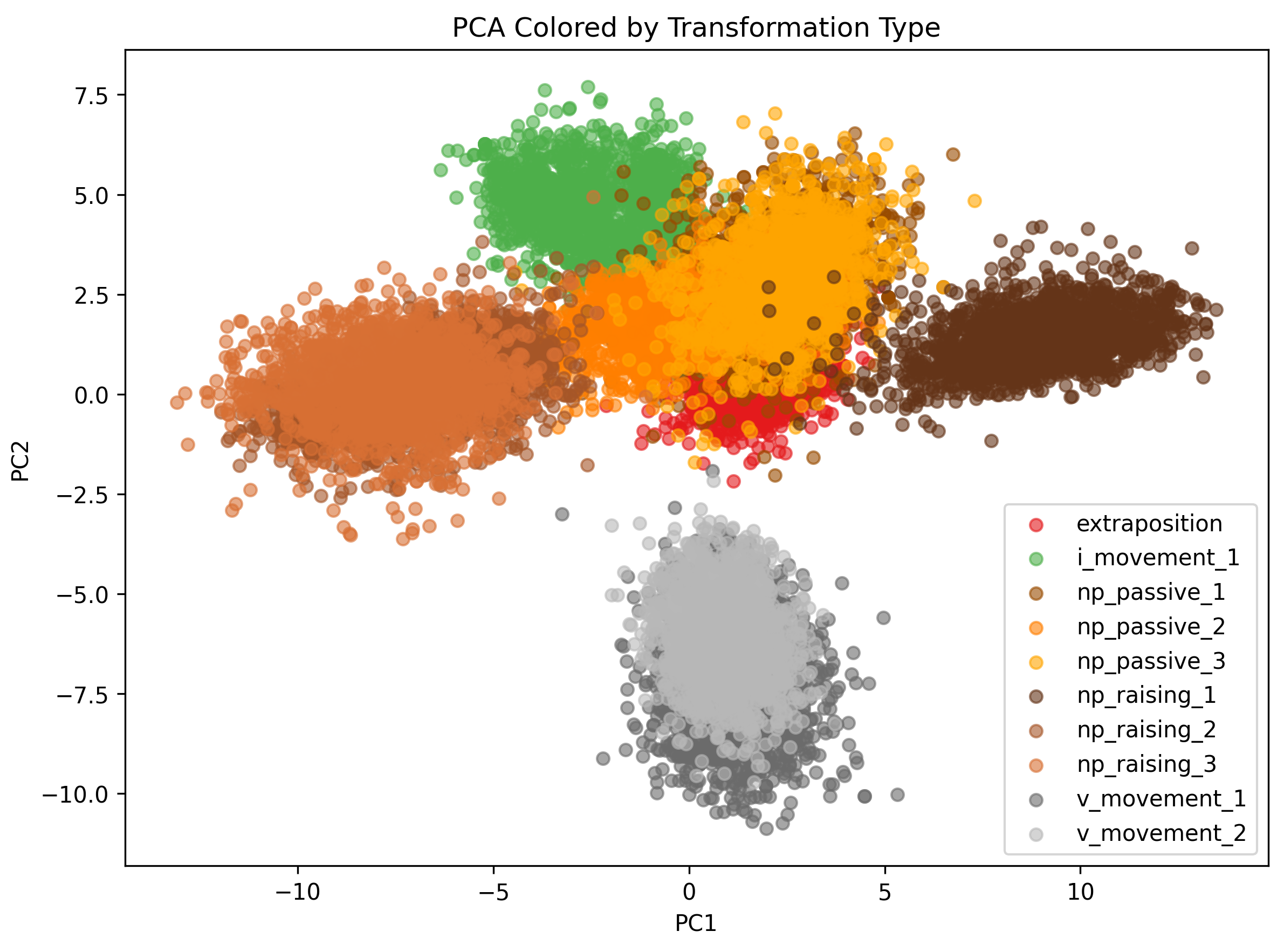}
    \caption{Step 32K}
    \label{fig:clustering_step32k}
\end{subfigure}
\hfill
\begin{subfigure}[t]{0.18\textwidth}
    \centering
    \includegraphics[width=\textwidth]{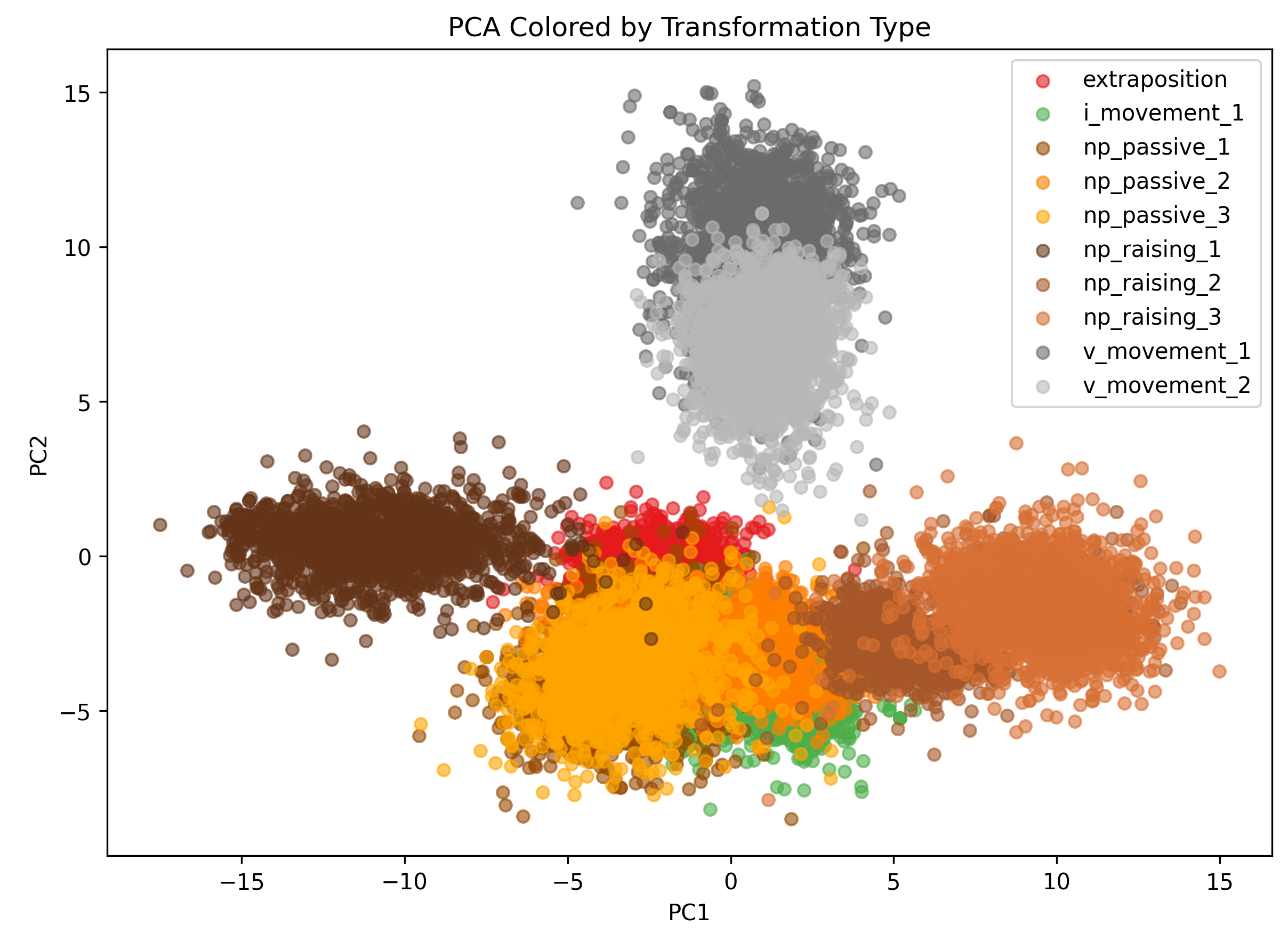}
    \caption{Step 64K}
    \label{fig:clustering_step64k}
\end{subfigure}
\hfill
\begin{subfigure}[t]{0.18\textwidth}
    \centering
    \includegraphics[width=\textwidth]{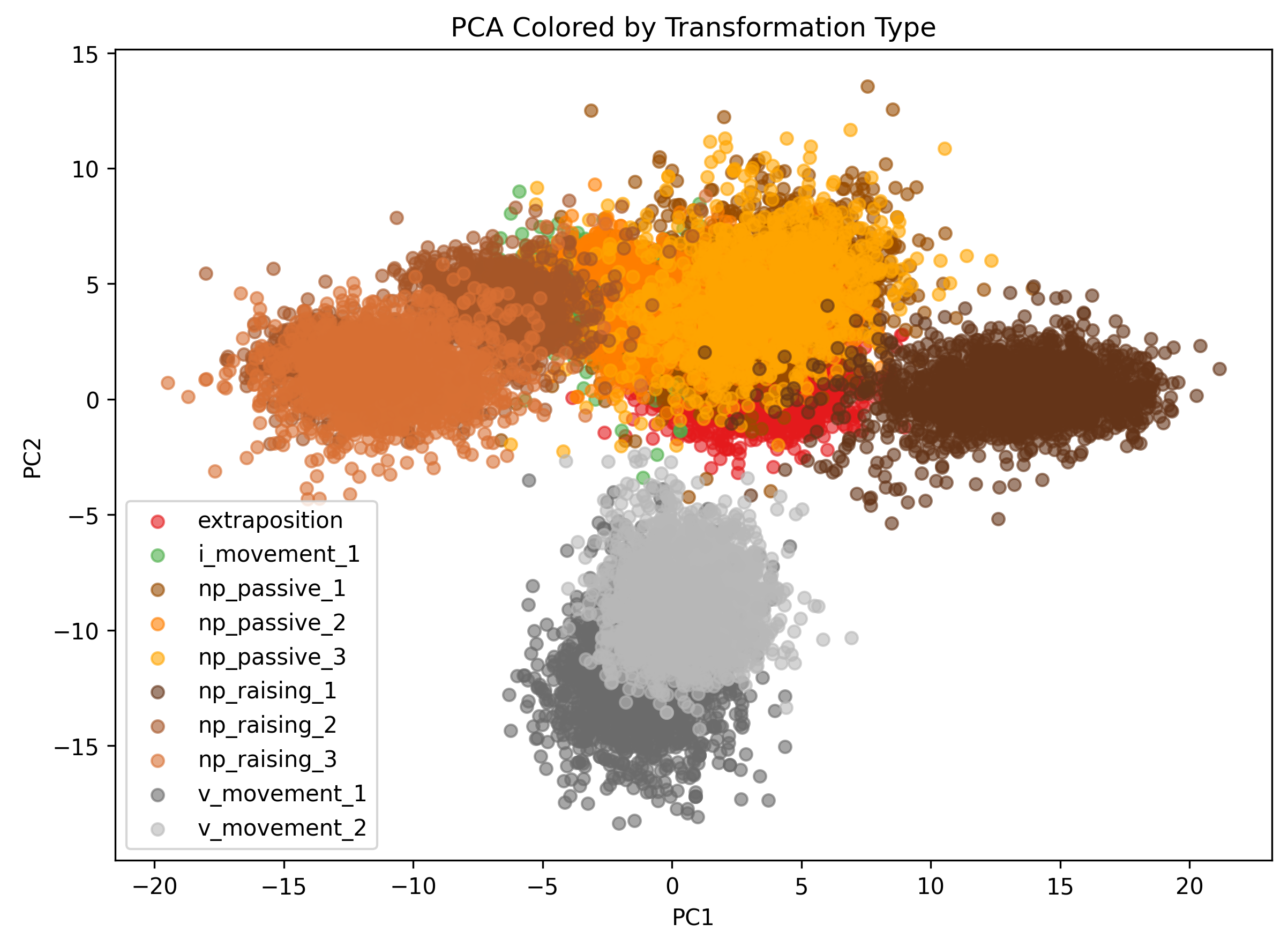}
    \caption{Step 96K}
    \label{fig:clustering_step96k}
\end{subfigure}
\hfill
\begin{subfigure}[t]{0.18\textwidth}
    \centering
    \includegraphics[width=\textwidth]{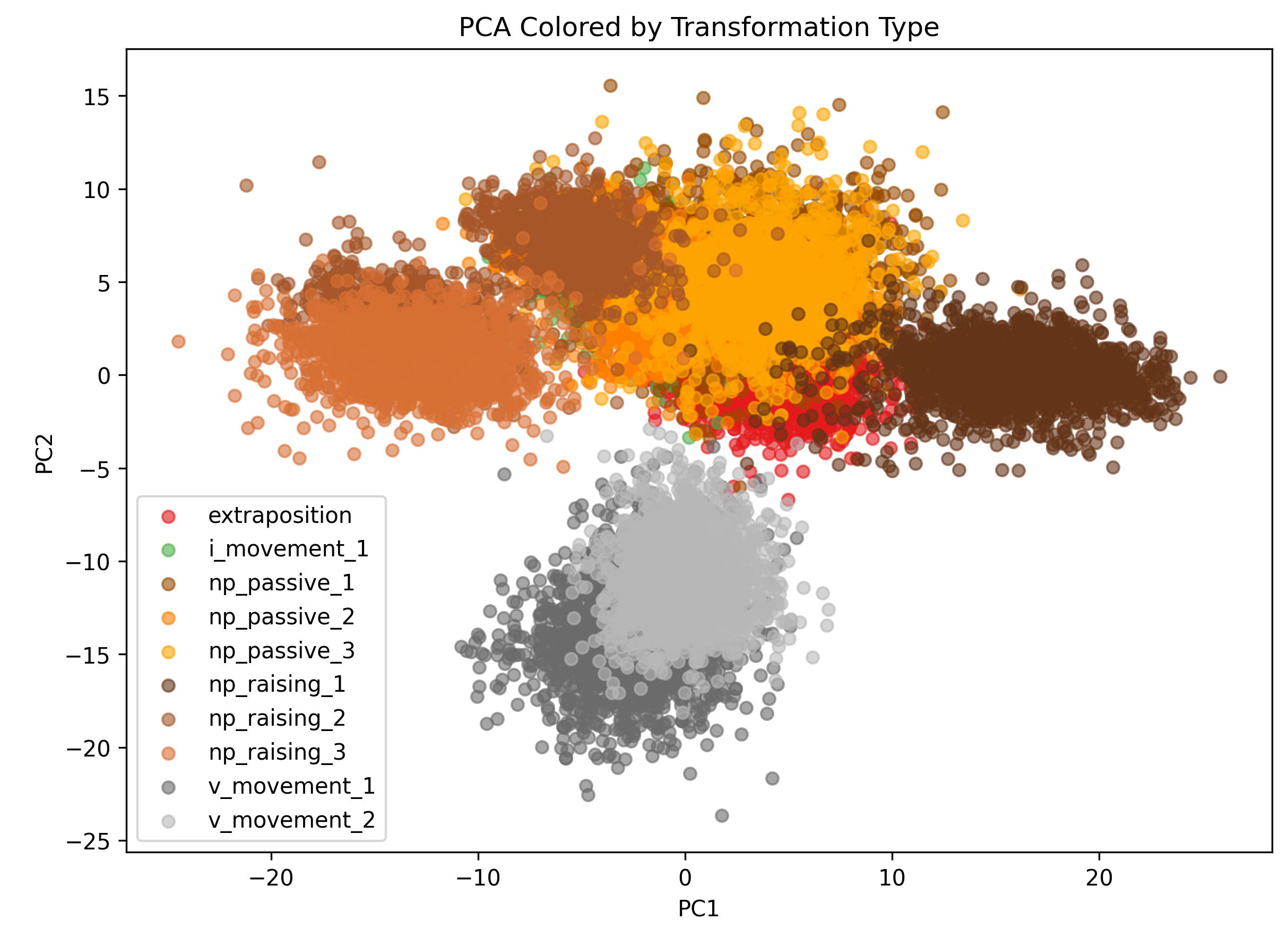}
    \caption{Step 128K}
    \label{fig:clustering_step128k}
\end{subfigure}

\vspace{0.5em}

\begin{subfigure}[t]{0.18\textwidth}
    \centering
    \includegraphics[width=\textwidth]{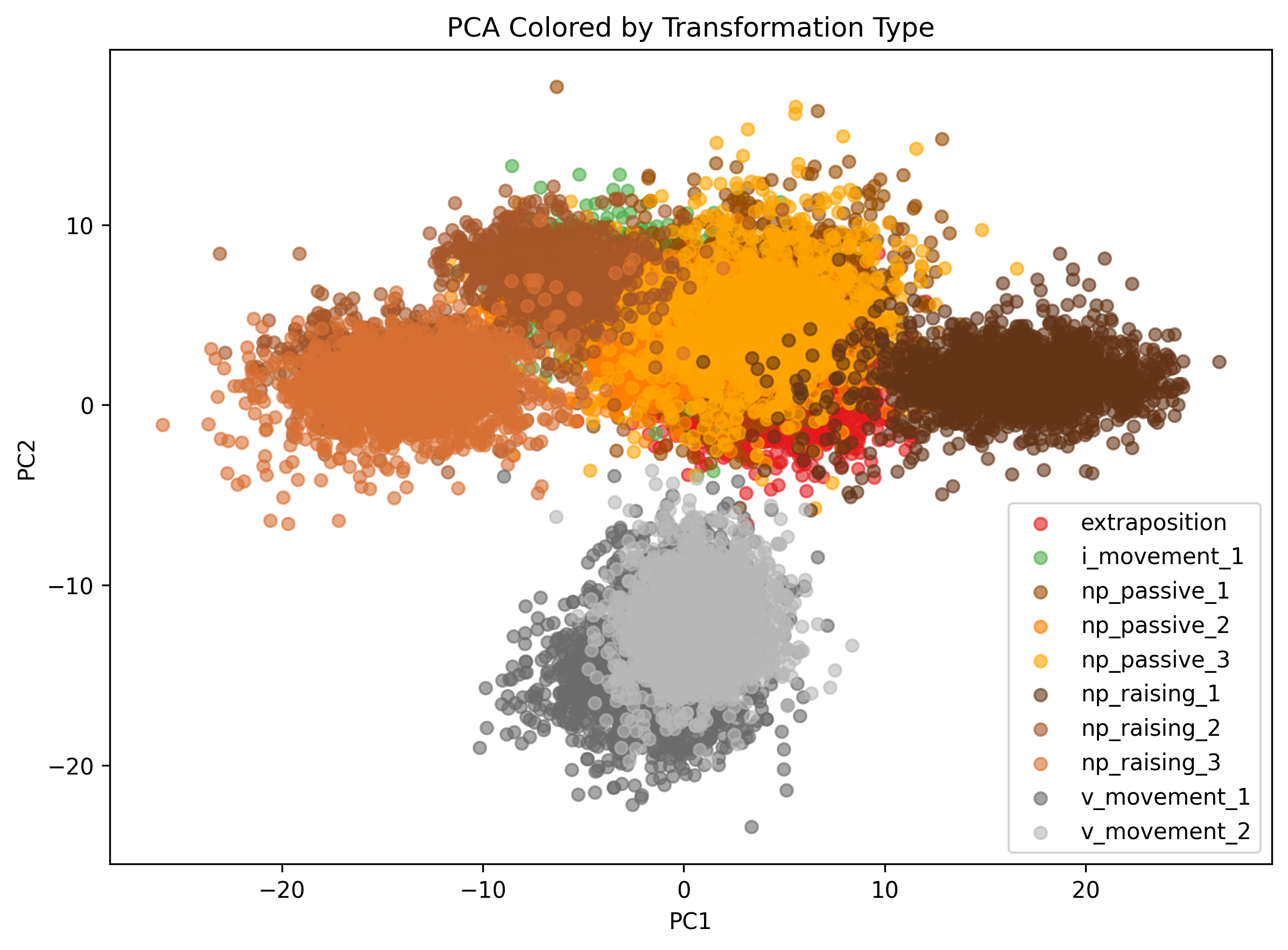}
    \caption{Step 143K}
    \label{fig:clustering_step143k}
\end{subfigure}

\caption{PCA clustering analysis of syntactic transformation embedding differences across Pythia-410M training. Each subplot shows the principal component analysis of embedding difference vectors for all transformation types at different training steps. Clustering patterns reveal how the model's internal representations of syntactic transformations evolve during training.}
\label{fig:pca_clustering_progression}
\end{figure*}

\begin{figure}[!htb]
    \centering
    \includegraphics[width=\columnwidth]{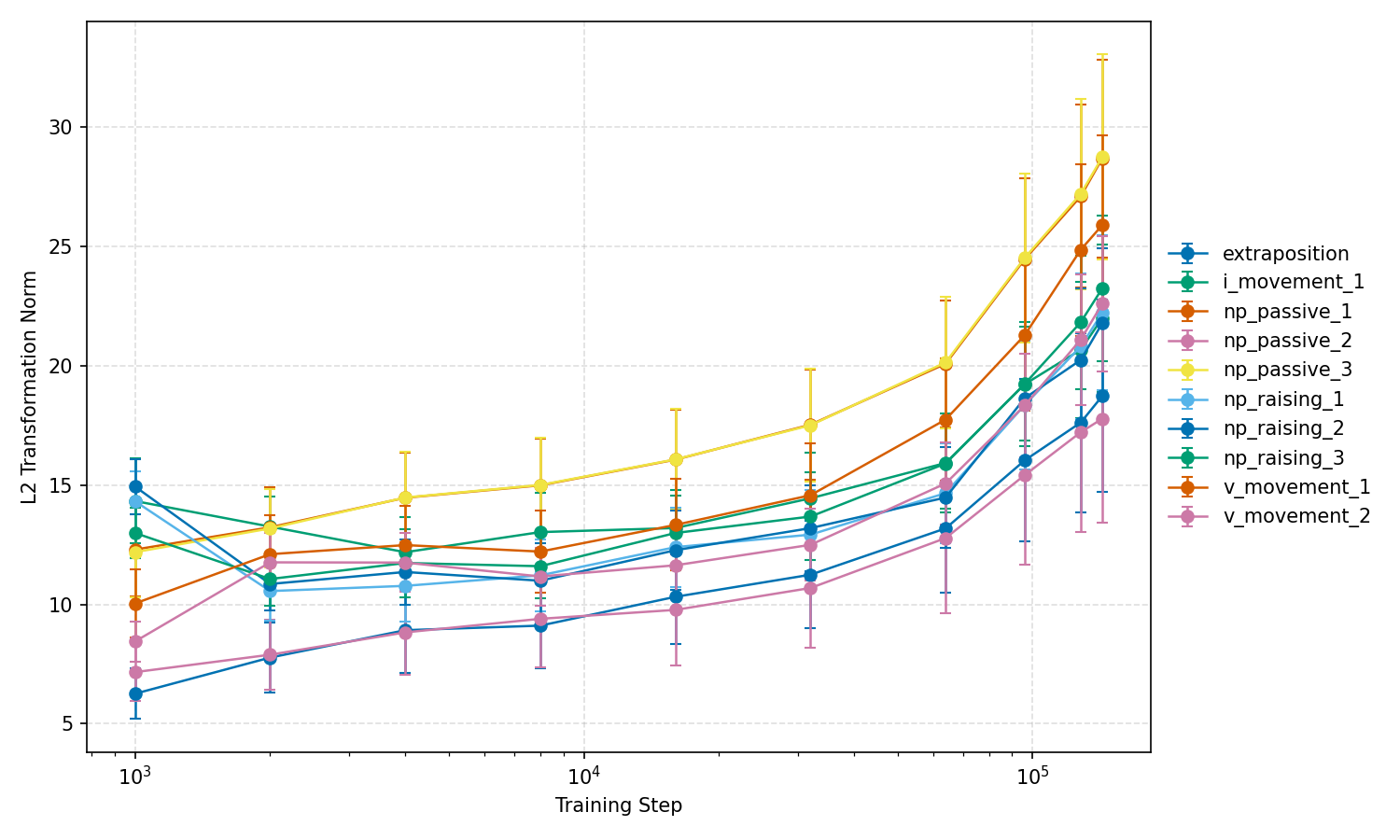}
    \vspace{0.3em} 
    \caption{L2 transformation norms for Pythia-410M training checkpoints.}
    \label{fig:pythia_l2_error_bars}
\end{figure}

\begin{table*}[!htb]
\centering
\caption{Description and examples for nested transformations. Nested transformations are combinations of two compatible single-level transformations}
\rowcolors{2}{gray!10}{white} 
\begin{tabularx}{\textwidth}{|l|X|}
\hline
\textbf{Transformation Sequence} & \textbf{Example} \\
\hline
\textbf{NP Passive 2 → I Movement} & "[empty] Put the corn on the table" → "The corn was put on the table" → "Was the corn put on the table?" \\
\textbf{NP Passive 3 → I Movement} & "The baker took the muffin away" → "The muffin was taken away by the baker" → "Was the muffin taken away by the baker?" \\
\textbf{NP Raising 1 → Extraposition} & "It seems that a review of my latest book appeared in the news" → "A review of my latest book appears to have appeared in the news" → "A review appears to have appeared in the news of my latest book" \\
\textbf{V Movement → I Movement} & "The children; to play outside" → "The children play outside" → "Do the children play outside?" \\
\textbf{Extraposition → NP Passive 1} & "The student from the university wrote the essay" → "The student wrote the essay from the university" → "The essay was written from the university by the student" \\
\textbf{NP Raising 3 → NP Passive 3} & "It seems that the chef placed the ingredients on the counter" → "The chef seems to place the ingredients on the counter" → "The ingredients seem to be placed on the counter by the chef" \\
\hline
\end{tabularx}
\label{tab:nested_transformations_examples}
\end{table*}

\end{document}